\documentclass[11pt]{article}

\usepackage{fullpage}
\usepackage{times,url,bm}
\usepackage{amsmath,amssymb,amsthm,mathtools}
\usepackage{xcolor}

\newif\ifshowgraphs
\showgraphstrue

\ifshowgraphs
  \usepackage{tikz}
  \usetikzlibrary{arrows.meta}
  \usepackage{pgfplots}
  \pgfplotsset{compat=1.18}
\fi
\usepackage{subcaption}
\usepackage{float}
\usepackage[square,numbers]{natbib}
\usepackage{hyperref}
\hypersetup{
  colorlinks = true,
  urlcolor   = blue,
  linkcolor  = blue,
  citecolor  = black
}
\newtheorem{theorem}{Theorem}[section]
\newtheorem{proposition}[theorem]{Proposition}
\newtheorem{lemma}[theorem]{Lemma}
\newtheorem{remark}[theorem]{Remark}

\newcommand{\R}{\mathbb{R}}
\newcommand{\relu}[1]{\left[#1\right]_+}
\newcommand{\ip}[2]{\left\langle #1,#2\right\rangle}
\newcommand{\norm}[1]{\left\lVert #1\right\rVert}
\newcommand{\Lip}{\operatorname{Lip}}

\newcommand{\Ucal}{\mathcal{U}}
\newcommand{\E}{\mathop{\mathbb E}}
\newcommand{\Prob}{\mathop{\mathbb P}}
\newcommand{\Dcal}{\mathcal{D}}
\newcommand{\Ocal}{\mathcal{O}}

\newcommand{\be}{\boldsymbol{e}}
\newcommand{\bone}{\boldsymbol{1}}
\newcommand{\bp}{\boldsymbol{p}}
\newcommand{\bq}{\boldsymbol{q}}
\newcommand{\bs}{\boldsymbol{s}}
\newcommand{\bu}{\boldsymbol{u}}
\newcommand{\bv}{\boldsymbol{v}}
\newcommand{\bX}{\boldsymbol{X}}
\newcommand{\bx}{\boldsymbol{x}}
\newcommand{\bw}{\boldsymbol{w}}
\newcommand{\by}{\boldsymbol{y}}
\newcommand{\bz}{\boldsymbol{z}}
\newcommand{\bxi}{\boldsymbol{\xi}}
\newcommand{\bzero}{\boldsymbol{0}}

\newcommand{\vertexcubesize}{0.06}
\newcommand{\drawreferencecube}[1][gray!65]{%
  \draw[#1,line width=0.45pt]
    (-1,-1,-1)--(1,-1,-1)--(1,1,-1)--(-1,1,-1)--cycle
    (-1,-1,1)--(1,-1,1)--(1,1,1)--(-1,1,1)--cycle
    (-1,-1,-1)--(-1,-1,1)
    (1,-1,-1)--(1,-1,1)
    (1,1,-1)--(1,1,1)
    (-1,1,-1)--(-1,1,1);
}

\newcommand{\drawvertexcube}[4][orange]{%
  \begin{scope}[shift={(#2,#3,#4)}]
    \def\a{\vertexcubesize}
    \path[fill=#1!16,draw=none]
      (-\a,-\a,\a)--(\a,-\a,\a)--(\a,\a,\a)--(-\a,\a,\a)--cycle;
    \path[fill=#1!27,draw=none]
      (\a,-\a,-\a)--(\a,\a,-\a)--(\a,\a,\a)--(\a,-\a,\a)--cycle;
    \path[fill=#1!38,draw=none]
      (-\a,-\a,-\a)--(\a,-\a,-\a)--(\a,-\a,\a)--(-\a,-\a,\a)--cycle;
    \draw[#1!75!black,line width=0.48pt]
      (-\a,-\a,-\a)--(\a,-\a,-\a)--(\a,\a,-\a)--(-\a,\a,-\a)--cycle
      (-\a,-\a,\a)--(\a,-\a,\a)--(\a,\a,\a)--(-\a,\a,\a)--cycle
      (-\a,-\a,-\a)--(-\a,-\a,\a)
      (\a,-\a,-\a)--(\a,-\a,\a)
      (\a,\a,-\a)--(\a,\a,\a)
      (-\a,\a,-\a)--(-\a,\a,\a);
    \fill[#1!80!black] (0,0,0) circle (0.75pt);
  \end{scope}
}

\title{Every Layer Counts:\\ An Exponential $L_2$ Depth Hierarchy
for ReLU Networks}
\author{Itay Safran\\
Stein Faculty of Computer and Information Science, Ben-Gurion University of the Negev\\
\href{mailto:safrani@bgu.ac.il}{\textcolor{black}{\texttt{safrani@bgu.ac.il}}}}
\date{}
\hypersetup{
  pdftitle={Every Layer Counts: An Exponential L2 Depth Hierarchy for ReLU Networks},
  pdfauthor={Itay Safran}
}

\begin{document}
\maketitle

\begin{abstract}
We prove a depth hierarchy for ReLU neural networks in which every additional
ReLU layer can save exponentially many neurons.  For all $k\geq2$, we
construct a globally $[0,1]$-valued, $1$-Lipschitz function realized by a
depth-$(k+1)$ network of width $\Ocal(d^4)$, whereas any
depth-$k$ network with
unrestricted weights and width at most
$
    \frac{2^d}{2d(k-1)}
$
has squared $L_2$ error at least $1/24$ under an absolutely continuous
distribution supported at exponential distance from the origin. To the best of our knowledge, this is the first exponential hierarchy across all adjacent fixed
depths, and the first exponential separation for ReLU networks
between two fixed depths whose shallower network has depth at least $3$.  The lower bound also immediately yields the corresponding hierarchy
for exact computation. Moreover, the case $k=2$ gives a compactly supported separation
between depths $3$ and $2$ with unrestricted shallow-network weights,
answering a question raised by
\citet[Sec.~2.3]{safran-eldan-shamir-2019}. The distribution used in our construction nevertheless has all its mass at exponential radius, placing the hierarchy outside the regularity regime in which such a separation would imply major threshold-circuit lower bounds.

We also prove an exact separation for a more regular target, which is globally
$[0,1]$-valued and $\Ocal(\sqrt d)$-Lipschitz and maps the unit hypercube
onto $[0,1]$.  It is computed by a polynomial-width depth-$4$ network,
whereas any depth-$3$ network agreeing with it on the unit hypercube requires
exponentially many first-layer neurons, even with unrestricted weights.

\end{abstract}

\section{Introduction}

Depth is a defining feature of modern neural-network architectures.  Beyond
its empirical importance, depth can fundamentally increase representational
efficiency: a function admitting a compact deep representation may require a
polynomially or even exponentially larger network at a smaller depth.  Two
seminal works established complementary forms of this phenomenon.
\citet{eldan-shamir-2016} proved an exponential separation of depth $3$ from
depth $2$, whereas \citet{telgarsky-2016} constructed highly compositional
functions separating networks whose depth grows with the problem parameters
from substantially shallower networks.  Subsequent work broadened these two
regimes and progressively extended the results to broader classes of target functions and data distributions \citep{daniely-2017,safran-shamir-2017,liang-srikant-2017,
yarotsky-2017,safran-eldan-shamir-2019,venturi-et-al-2022,safran-reichman-valiant-2025}.  The resulting literature nevertheless remained
divided between two regimes: exponential separations of depth $3$ from depth
$2$, and exponential separations between a fixed shallow depth and a deeper
architecture whose depth grows with the dimension or the desired accuracy.
Recent work established a superlinear depth hierarchy for the exact
computation of the maximum function \citep{Safran2026}, but did not yield an
exponential separation.  Prior to this work, no exponential separation was
known between any two fixed depths $k'>k\geq3$ for ReLU networks or, to the
best of our knowledge, for networks using any other standard nonpolynomial
activation.

The analogous question has long been central in circuit complexity.
Threshold circuits are essentially feedforward neural networks with threshold
activations, operating on Boolean inputs and producing Boolean outputs.  A
central goal in this setting is a full depth hierarchy: for any constant
$k$, one would like to exhibit an explicit function computable by a
polynomial-size depth-$(k+1)$ threshold circuit
but requiring exponential size at depth $k$.
Even substantially weaker goals remain open.  A separation of
depth $3$ from depth $2$ is known when the weights of the depth-$2$ circuit are
polynomially bounded \citep{hajnal-et-al-1993}, but the corresponding result
for arbitrary weights is unknown.  Separating any larger constant depth from
depth $3$ is likewise a longstanding open problem.  For lower bounds against depth $4$ and above, there is a further obstacle. \citet{razborov-rudich-1997} identified a broad class of arguments, called natural proofs, that encompasses most known circuit lower-bound techniques. They showed that, assuming the existence of pseudorandom functions, such arguments cannot prove superpolynomial lower bounds against a circuit class capable of computing them. Since candidate pseudorandom-function families can be computed by polynomial-size depth-$4$ threshold circuits \citep{krause-lucks-2001,naor-reingold-2004}, the natural-proofs barrier applies whenever the shallower circuit has depth at least $4$.

The contrast between the two models is particularly striking because neural
networks already admit separations whose threshold-circuit analogs remain
open.  The construction of \citet{telgarsky-2016} separates networks of
polynomially growing depth from constant-depth networks, whereas the
corresponding separation of
polynomial-depth threshold circuits from constant-depth threshold circuits is
unknown.  Similarly, the lower bound of \citet{eldan-shamir-2016} separating depths
$3$ and $2$ places no restriction on the weights of the depth-$2$ network,
while the corresponding arbitrary-weight separation for threshold circuits
remains open.  Nevertheless, the central fixed-depth gap
was shared by both models.  For a decade following the results of
\citet{eldan-shamir-2016} and \citet{telgarsky-2016}, no exponential
separation was known for fixed depths once the shallower depth was at least $3$, even when the two depths were allowed to differ by more than one layer.

\citet{vardi-shamir-2020} formalized the connection between this gap and the
corresponding circuit-complexity barriers.  Under certain boundedness and
regularity assumptions on the target function and input distribution, they
showed that, for all $k\geq4$, an exponential separation of some larger
constant ReLU depth from depth $k$ would imply a separation of some larger
constant threshold-circuit depth from depth $k-2$.  Thus, the case $k=4$
would resolve the arbitrary-weight depth-$2$ threshold-circuit problem; the
case $k=5$ would separate a larger constant depth from depth $3$; and the
cases $k\geq6$ would enter the regime governed by the natural-proofs barrier
\citep{razborov-rudich-1997}. These implications serve both as motivation and as a boundary marker: a sufficiently regular constant-depth hierarchy
for ReLU networks would resolve problems extending far beyond neural-network
approximation.

Our main result resolves the fixed-depth gap for ReLU networks and, more
strongly, establishes a complete exponential hierarchy across all adjacent
depths.

\begin{theorem}[Informal version of Theorem~\ref{thm:l2-hierarchy}]
For all fixed $k\geq2$, there exist families of functions and corresponding
data distributions such that the functions are represented by
polynomial-width depth-$(k+1)$ ReLU networks, whereas approximating them to
error below a fixed constant with depth-$k$ networks, under the corresponding
distributions, requires width exponential in the input dimension, even with
unrestricted weights.
\end{theorem}

The proof of our main result relies crucially on a distribution whose support
has exponential radius.  This violates the
\emph{almost-bounded-support} assumption of
\citet{vardi-shamir-2020}, which informally requires most of the probability
mass to lie within polynomial distance of the origin.  Hence, to the best of
our knowledge, our approximation hierarchy has no immediate consequences for
threshold circuits.  Since the support radius of the distribution, the
target's Lipschitz parameter, and the required approximation accuracy can be
traded against one another by rescaling (see
\citet[Theorem~9]{safran-eldan-shamir-2019} for a formal statement), our
construction requires at least one of them to scale exponentially with the
dimension.

The use of an exponential scale is not uncommon in the literature.  The
separation of \citet{telgarsky-2016} also relies crucially on such a scale:
although its domain is the unit interval, the Lipschitz constant of the target
grows exponentially with the composition depth.  The arbitrary-weight
separation of \citet{eldan-shamir-2016} instead uses a heavy-tailed
distribution and therefore has no finite support radius.  Its hard target is
supported in a ball of radius $\Ocal(\sqrt d)$, and all but an
inverse-polynomial amount of probability mass lies within a polynomial
radius.  In our chosen normalization, by contrast, the entire distribution
is supported at exponential distance from the origin (see
Remark~\ref{rem:compact-support-unrestricted-weights}).  We use this normalization for expositional simplicity:
keeping the target $1$-Lipschitz places the exponential scale entirely in the
support radius and makes it straightforward to compare the construction with
the assumptions of the Vardi--Shamir reduction.  Equivalently, one may
rescale the input so that the distribution has constant support radius, at
the cost of making the target's Lipschitz constant exponential, without
changing the depth, width, or approximation error of the separation.

It remains open whether the bounds on the parameters in our $L_2$ hierarchy can be reduced to a polynomial, or even a constant, while retaining an exponential gap between the required network size.  Since the Lipschitz and accuracy scales in our theorem are already constant and the other relevant regularity properties are retained, reducing its support radius to a polynomial would bring the result into the regime of the Vardi--Shamir reduction.  Any such improvement would therefore either have to overcome the corresponding circuit-complexity barrier or relax a different assumption.

Exact computation provides an alternative setting in which the approximation
scale is removed altogether. Rather than bounding average error under a
chosen distribution, one requires the network to agree with the target at every point of the computation domain.  Although average-case approximation is more directly related to learning, exact computation removes the distribution from the problem and imposes a stricter requirement on the competing network, which can make representational lower bounds easier to prove.

This contrast is already visible at depth $2$.  Known approximation lower bounds at this depth require technically involved arguments \citep{eldan-shamir-2016,daniely-2017,safran-reichman-valiant-2025}, whereas an elementary geometric argument shows that the maximum of three inputs cannot be represented by any finite depth-$2$ ReLU network \citep{mukherjee2017lower,Safran2026}.  The maximum function is therefore a natural testbed for exact representation: it is easily computed by a logarithmic-depth tournament that takes pairwise maxima at each layer \citep{arora2016understanding}, yet its complexity at smaller depths remains poorly understood.  The exact representation of the maximum, and of continuous piecewise-linear (CPWL) functions more generally, has motivated a growing line of work
\citep{mukherjee2017lower,hertrich2021towards,haase2023lower,
bakaev2025depth,averkov2025expressiveness,grillo2025depth,
bakaev2025better}. Despite recent progress, it is still unknown whether the maximum of $d$ inputs can be expressed by a depth-$3$ ReLU network for $d\geq11$. Recent work established a superlinear depth hierarchy for the exact computation of the maximum \citep{Safran2026}, which in particular implies a quadratic lower bound for depth $3$, but did not give an exponential separation between fixed depths.  Our second result obtains such a separation between depths $4$ and $3$, but for a CPWL target different from the maximum.

\begin{theorem}[Informal version of Theorem~\ref{thm:exact-separation}]
There is a family of $\Ocal(\sqrt{d})$-Lipschitz functions, one in each dimension, that can be computed
exactly by polynomial-width depth-$4$ ReLU networks, but for which any
depth-$3$ network agreeing with the target on the unit hypercube requires
first-layer width exponential in the dimension, even with unrestricted
weights.
\end{theorem}

It remains open whether every CPWL function in $d$ dimensions can be
represented by a depth-$3$ ReLU network. This qualitative universality
question, however, does not settle the efficiency of such representations.
Even if the answer is affirmative, the theorem above shows that depth-$3$
representations can require exponential width for functions having
polynomial-width depth-$4$ representations. Thus, independently of whether
depth $3$ can represent every CPWL function, there is an exponential
quantitative gap between the two depths.

Although our main $L_2$ hierarchy immediately yields exact separations for all adjacent depths, those separations inherit the exponential scale of the construction.  Theorem~\ref{thm:exact-separation} shows that an exponential gap can also arise in a much more regular geometric setting: the lower bound holds on the unit hypercube for a globally bounded target with a polynomial Lipschitz constant and a non-degenerate output range. Section~\ref{sec:exact-separation} makes this comparison quantitative.

The remainder of the paper is organized as follows.  After reviewing related
work below, we introduce notation and conventions in
Section~\ref{sec:preliminaries}.
Section~\ref{sec:l2-hierarchy} proves the adjacent-depth $L_2$ hierarchy, and
Section~\ref{sec:exact-separation} proves the exact separation between depths $4$ and $3$ for a more regular target. Section~\ref{sec:conclusions} summarizes our results and discusses directions for future work.  The formal proofs are deferred to the appendices.

\begin{remark}[Compact support and unrestricted weights]
\label{rem:compact-support-unrestricted-weights}
The case $k=2$ answers the question, raised by
\citet[Sec.~2.3]{safran-eldan-shamir-2019} and restated by
\citet{safran-reichman-valiant-2025}, of whether an exponential separation between depths $3$ and $2$ can hold with respect to a compactly
supported distribution without any restriction on the shallow network's
weights.  The latter work obtained a separation on the unit ball with
constant Lipschitz and accuracy scales, but under an exponential upper bound
on the shallow network's weights. Thus, it kept the domain, Lipschitz, and accuracy scales constant while leaving open the combination of compact support and
unrestricted weights.  Our hierarchy imposes no restriction on the weights
at any depth, and its $k=2$ specialization gives
such an exponential lower
bound under an absolutely continuous, compactly supported distribution.

This conclusion cannot be obtained simply by truncating the heavy-tailed
distribution of \citet{eldan-shamir-2016}: its tail probability decreases
only polynomially with the truncation radius, while an unrestricted ReLU
network may have an arbitrarily large linear growth rate.  Its squared
approximation error may therefore be concentrated entirely in the discarded
tail.  Establishing hardness directly under compact support is thus a genuine
strengthening, rather than merely a truncation of the earlier result.
Although our result answers the question posed by
\citet{safran-eldan-shamir-2019} in its stated form, it does so by trading the
earlier restriction on the shallow network's weights for a distribution supported at exponential radius. It therefore leaves open the quantitatively stronger goal of obtaining the same separation with support radius polynomial in $d$, or ideally bounded by an absolute constant.

\end{remark}

\subsection{Related work}

\paragraph{Separating depth 3 from depth 2}

The continuous $L_2$ literature contains several separations between depths $3$ and $2$, distinguished by the geometry of the target, the associated scale parameters, and whether the shallow network's weights are
restricted.  \citet{eldan-shamir-2016} proved the first exponential
separation using an approximately radial, oscillatory target and a Fourier
analysis argument.  Their lower bound permits unrestricted weights, but uses
a heavy-tailed distribution.  \citet{venturi-et-al-2022}
extended this Fourier-analytic approach beyond radial functions to product distributions.
\citet{daniely-2017} instead used spherical harmonics to obtain a compactly
supported separation on a product of spheres, under an exponential upper
bound on the shallow network's weights.  Reductions and refinements of these analytic
methods yielded lower bounds for ball and ellipsoid indicators, other
non-oscillatory radial targets, and targets that can be learned by deeper
networks
\citep{safran-shamir-2017,safran-eldan-shamir-2019,
safran-lee-2022,nichani-damian-lee-2023}.  Most recently,
\citet{safran-reichman-valiant-2025} obtained an
exponential separation on the unit ball with constant Lipschitz and accuracy
scales, also under an exponential bound on the shallow network's weights.
Complementarily, \citet{hsu-et-al-2021} showed that, under the uniform
distribution on $[-1,1]^d$, any $\Ocal(1)$-Lipschitz target admits a
polynomial-width depth-$2$ approximation at any fixed constant accuracy,
underscoring the role of the approximation distribution.  The
$k=2$ case
of Theorem~\ref{thm:l2-hierarchy} instead permits unrestricted weights and
retains compact support, at the cost of placing that support at exponential
radius and concentrating the mass on a union of widely separated cubes.

Despite their different settings, the preceding exponential lower bounds are
all tailored to ruling out a single hidden nonlinear layer, and none has
yielded an exponential lower bound once the shallower network has depth
greater than $2$.  Our proof instead uses a recursive mechanism that
reduces approximation at one depth to approximation of the preceding target.
Because the same reduction can be applied once per layer, it yields
separations for all pairs of adjacent depths rather than only between depths
$3$ and $2$.

\paragraph{Deeper separations}

A second line of work obtains exponential benefits from depth by allowing the
depth of the efficient network to
vary with the problem parameters.  The iterated-sawtooth construction of
\citet{telgarsky-2016} is computable by constant-width networks with
$\Theta(k^2)$ layers, but cannot be approximated
at substantially smaller depth
$\Ocal(k)$ without
$\Omega(2^{k})$ neurons. Concurrent work by \citet{liang-srikant-2017}, \citet{yarotsky-2017}, and \citet{safran-shamir-2017} established related depth advantages for smooth,
piecewise-smooth, and other natural targets, contrasting fixed-depth networks
with networks whose depth grows with the desired accuracy and obtaining
improved dependence on the inverse accuracy.  Further growing-depth
separations and approximation-rate results
were obtained by \citet{arora2016understanding} and
\citet{yarotsky-2018}, respectively.  These results provide strong
separations between fixed-depth and growing-depth regimes, but not exponential
separations between two fixed depths.

Most recently, \citet{krishnan-mossel-2026} obtained an all-depth tradeoff for
exact computation on Boolean inputs.  Their product-map family is computed by
a constant-width ReLU network of depth $n+1$, whereas
any depth-$r$, width-$w$
network computing it exactly
satisfies $w^r=\Omega(2^n)$.
Consequently, polynomial width is ruled out when
$r=o(n/\log n)$.  This result is again
fundamentally a growing-depth
separation and, moreover, concerns exact, infinite-accuracy computation on a
Boolean domain rather than average-case approximation under a continuous
distribution.  At the adjacent depth
$r=n$, their tradeoff gives only
a constant lower bound on the width, and it does not separate any two fixed
depths.  Thus, neither this result nor the earlier constructions discussed here supply an exponential separation between fixed depths
$k'>k\geq3$.  Theorem~\ref{thm:l2-hierarchy} fills precisely this gap, and
does so simultaneously for all adjacent pairs.

\paragraph{Exact computation}

Exact representation of CPWL functions removes the approximation parameter,
but unconditional depth lower bounds remain elusive when the width and the
real-valued weights are unrestricted.  \citet{mukherjee2017lower} showed that
the maximum of three coordinates cannot be represented at depth $2$, while
\citet{hertrich2021towards} initiated a systematic study of whether the
classes of CPWL functions representable with unrestricted width strictly grow
with depth.  Subsequent lower bounds for the maximum and for general CPWL
functions have required additional assumptions, such as integral or rational
weights, monotonicity or input-convexity, or the requirement that all intermediate ReLU activation boundaries conform to a prescribed polyhedral structure \citep{haase2023lower,averkov2025expressiveness,bakaev2025depth,
grillo2025depth}.  On the upper-bound side, the conjecture of
\citet{hertrich2021towards} that the maximum of five coordinates is not
representable at depth $3$ was refuted by \citet{bakaev2025better}, who
constructed such a representation and improved the general depth bound for
the maximum of $d$ coordinates.

Recent work strengthens these upper bounds. \citet{ruess-et-al-2026} proved that the maximum of $d$ coordinates is
representable with two hidden layers for all $d\leq10$, and that every CPWL function on $\mathbb{R}^d$ has such a representation for $d\leq9$.  Their structured
representation of the maximum of ten coordinates also yields a recursive
construction with fewer layers in arbitrary dimension.
This progress makes an unrestricted incomputability result at depth $3$ still
more delicate and directly motivates a quantitative alternative: even if
depth-$3$ networks can represent every CPWL function, how wide must they be? Theorem~\ref{thm:l2-hierarchy} implies an exponential separation for all
adjacent depths with exponential-radius scaling, while
Theorem~\ref{thm:exact-separation} answers this question with an exponential
lower bound on the unit hypercube for a target whose Lipschitz constant is polynomial in $d$. These complement the superlinear exact hierarchy for the maximum in \citet{Safran2026}, which applies through depths of order $\log\log d$ but does not give an exponential separation. \citet{grillo-montufar-2026} established a generic depth hierarchy:
for an open set of parameters, the realized functions admit no shallower
representation with generic parameters, regardless of width.  This does not
exclude nongeneric shallower representations and therefore yields neither an
unconditional depth separation nor a quantitative size lower bound. Lastly, our proof builds on a first-layer-collapse principle first used by
\citet{safran-reichman-valiant-2024} for approximation of the maximum and
subsequently refined by \citet{Safran2026} for exact computation. Here we deploy the principle differently: local depth-$2$ hardness forces some first-layer neuron
to switch between its active and inactive regions at each of exponentially many
perturbed hard points. The perturbation ensures
that any one neuron can account for only a tiny fraction of these points, yielding the
width lower bound.

\paragraph{Relations to circuit complexity}

Several works transfer ideas between neural-network lower bounds and Boolean
or communication complexity. \citet{martens-et-al-2013} used known lower bounds for depth-$2$ threshold circuits computing inner product modulo $2$ to derive an exponential representational lower bound for restricted Boltzmann machines, under an exponential upper bound on the magnitude of their weights. For Boolean inputs, \citet{mukherjee2017lower} used techniques from circuit and communication complexity to prove sublinear size lower bounds for ReLU networks. \citet{vardi-et-al-2021} subsequently used Boolean and communication complexity to prove nearly linear size lower bounds for benign real-valued functions.  In the other direction,
\citet{safran-reichman-valiant-2025} combined a depth-$2$ threshold-circuit lower bound with a worst-to-average-case reduction to obtain a continuous $L_2$ depth separation on the unit ball, under an exponential bound on the shallow network's weights.

\section{Preliminaries and Notation}
\label{sec:preliminaries}

\paragraph{Conventions.}
For a positive integer $n$, let $[n]\coloneqq\{1,\ldots,n\}$. We denote vectors using bold-faced letters (e.g.\ $\bx$). We write
$\ip{\bx}{\bz}$ for the Euclidean inner product, $\norm{\bx}_p \coloneqq \left(\sum_{i=1}^d |x_i|^p\right)^{1/p}$
for the $\ell_p$ norm,
$\norm{\bx}_\infty \coloneqq \max_{i\in[d]} |x_i|$ for the infinity norm, and $\norm{\bx}$ for the Euclidean norm. Matrices are denoted by capital letters.  The notation $\Lip(f)$ denotes the
Euclidean-to-Euclidean Lipschitz constant of $f$. A continuous map is continuous piecewise-linear (CPWL) if it is affine on each cell of some finite polyhedral partition of its domain.  For a first-layer neuron
$\bx\mapsto\relu{\ip{\bw}{\bx}+b}$ with $\bw\neq\boldsymbol{0}$, we call
$
    \{\bx\in\R^d:\ip{\bw}{\bx}+b=0\}
$
its activation hyperplane. A finite subset of $\R^d$ is in general position if no hyperplane contains more than $d$ of its points. Constants hidden by standard asymptotic notation $\Ocal(\cdot),\Omega(\cdot),\Theta(\cdot)$ are universal unless stated otherwise.

\paragraph{Neural networks.}
For $t\in\R$, the rectified linear unit (ReLU) is
\[
  \relu{t}\coloneqq\max\{0,t\},
\]
and is applied coordinatewise to vector arguments.  We consider fully
connected, feedforward networks.  A
depth-$k$ ReLU network
$N:\R^{d}\to\R^{n_{k}}$ has the form
\[
  N
  =A_{k}\circ\relu{\cdot}\circ
  A_{k-1}\circ\cdots\circ
  \relu{\cdot}\circ A_1,
\]
where each $A_j(\bz)=W_j\bz+\boldsymbol{b}_j$ is affine, and $W_j,\boldsymbol{b}_j$ are the weights and biases of the $j^\mathrm{th}$ layer, respectively. The depth of the network is the number of hidden layers plus one. The hidden
layer widths are
$n_1,\ldots,n_{k-1}$, the width of the
network is their
maximum, and its size is their sum. Unless stated otherwise, all weights and
biases are arbitrary real numbers.

\paragraph{Approximation and exact computation.}
For a probability distribution $\Dcal$ on $\R^d$ and scalar-valued functions
$f,N:\R^d\to\R$, we define the squared $L_2(\Dcal)$ error of $N$ relative to
$f$ by
\[
  \norm{N-f}_{L_2(\Dcal)}^2
  \coloneqq
  \E_{\bx\sim\Dcal}
  \left[\left(N(\bx)-f(\bx)\right)^2\right].
\]
Throughout the paper, $L_2$ error refers to this squared quantity, equivalently
the expected squared loss. For a set $\Omega\subseteq\R^d$, a network $N$
computes $f$ exactly on $\Omega$ if $N(\bx)=f(\bx)$ for all
$\bx\in\Omega$, and computes $f$ globally if this equality holds throughout
$\R^d$. In the exact-computation setting, we use \emph{computes},
\emph{represents}, and \emph{realizes} interchangeably. For a measurable set
$A\subseteq\R^d$ of positive finite volume, $\Ucal(A)$ denotes the uniform
probability distribution on $A$.

\section{\texorpdfstring{Adjacent-Depth $L_2$ Separation}{Adjacent-Depth L2 Separation}}
\label{sec:l2-hierarchy}

In this section, we introduce the main result of the paper.  It establishes a
depth hierarchy for all adjacent depths: at every fixed depth, adding one ReLU
layer can reduce the width required for constant-error approximation from
exponential to polynomial in the dimension.

\begin{theorem}
\label{thm:l2-hierarchy}
There is a universal constant $C>0$ such that, for all $d\geq2$
and $k\geq2$, there exist a globally defined
CPWL function
\[
  f_{k,d}:\R^d\to[0,1]
\]
and an absolutely continuous probability distribution
$\Dcal_{k}^d$ such that the following hold.
\begin{enumerate}
  \item
  The function
  $f_{k,d}$ is computed globally by a
  depth-$(k+1)$ ReLU network of width at most
  $Cd^4$.

  \item
  If a depth-$k$ ReLU network $N$ has
  width at most
  \[
    \frac{2^d}{2d(k-1)},
  \]
  then
  \[
    \E_{\bx\sim\Dcal_{k}^d}
    \left[\left(N(\bx)-f_{k,d}(\bx)\right)^2\right]
    \geq\frac1{24}.
  \]
\end{enumerate}
\end{theorem}

For universal constants $c,C>0$, the construction given in
Appendix~\ref{app:l2-hierarchy-proof} has the following additional
properties.  The target is globally $1$-Lipschitz.  The distribution is
uniform on the union of
$2^{d(k-1)}$ pairwise disjoint
axis-aligned open cubes, each of side length
$2\cdot4^{k-1}$.  Its
support lies outside a Euclidean ball of radius
$2^{c(k-1)d^3}$ and is contained in a
Euclidean ball of radius
$2^{C(k-1)d^3}$.

Previously, no exponential lower bound separated two fixed depths
$k'>k\geq3$ for ReLU networks or, to the best of our knowledge, for networks
using any other standard nonpolynomial activation.
Theorem~\ref{thm:l2-hierarchy} gives such a separation for all pairs of
adjacent depths, and therefore a complete adjacent-depth hierarchy for ReLU
networks.

\paragraph{Relation to the Vardi--Shamir barrier.}
For the adjacent depths covered by the reduction of
\citet{vardi-shamir-2020}, the assumption violated by our construction is
almost-bounded support.  Their reduction requires that, for all inverse-polynomial choices of $\delta$, all but $\delta$ of the probability mass lie in
a box $[-R,R]^d$ for some $R=\operatorname{poly}(d)$.  By contrast, for all
fixed $k\geq2$, the entire support of
$\Dcal_{k}^d$ lies outside a
Euclidean ball of radius
\[
    2^{c(k-1)d^3}.
\]
Consequently, for sufficiently large $d$, no box of polynomial radius
contains any mass, and their reduction to threshold-circuit lower bounds does
not apply.  This is the precise reason our hierarchy avoids the
Vardi--Shamir barrier.  For completeness, the other relevant conditions are
satisfied: the
target is $[0,1]$-valued, and, for almost all fixed values of all but one
coordinate, the conditional density of the remaining coordinate is at most
$\frac18$.

\subsection{Construction and proof technique}

At a high level, our proof proceeds by induction and recursively creates
$2^d$ copies of the preceding hard function using an additional hidden layer. A polynomial-width transformation places the centers of the copies in general position, and the copies are made sufficiently small that no single hyperplane can intersect more than $d$ of them. Ensuring
this property requires rescaling the copies by an exponentially small factor,
which accounts for the exponential scale of our construction.  Consequently,
the activation hyperplanes in the first hidden layer of a narrow network
intersect only a small fraction of the copies.  On all remaining copies, that
layer has a fixed activation pattern and can be merged with the next layer, yielding a shallower network on the preceding hard function.
Iterating this argument yields the
hierarchy.  After the recursion, a final input rescaling expands all the
copies to normalize the target's Lipschitz constant; this expansion moves the
resulting approximation domain to exponential radius.

The proof has three main steps.

\paragraph{Step 1: Perturbing the Boolean hypercube.}
Lemma~\ref{lem:general-position-points}, which will also be used for the exact
separation, constructs a uniformly small CPWL displacement $\Delta$ and
perturbed vertices
\[
  \bq_{\bs}=\bs+\Delta(\bs),
  \qquad
  \bs\in\left\{-\frac12,\frac12\right\}^d.
\]
The map $\Delta$ is computed by a single-hidden-layer ReLU network of width
$\Ocal(d^4)$, and, for any set of $d+1$ perturbed vertices, the determinant
of their augmented coordinate matrix has magnitude at least
$2^{-\Ocal(d^3)}$. Because these determinants are nonzero, the augmented coordinate vectors are linearly independent, and hence the perturbed vertices are in general
position. The lower bound on their absolute values further allows us to place
sufficiently small cubes around the vertices so that no hyperplane intersects
more than $d$ of the cubes. The geometric role of this perturbation is illustrated
in Figure~\ref{fig:boolean-cube-perturbation}.

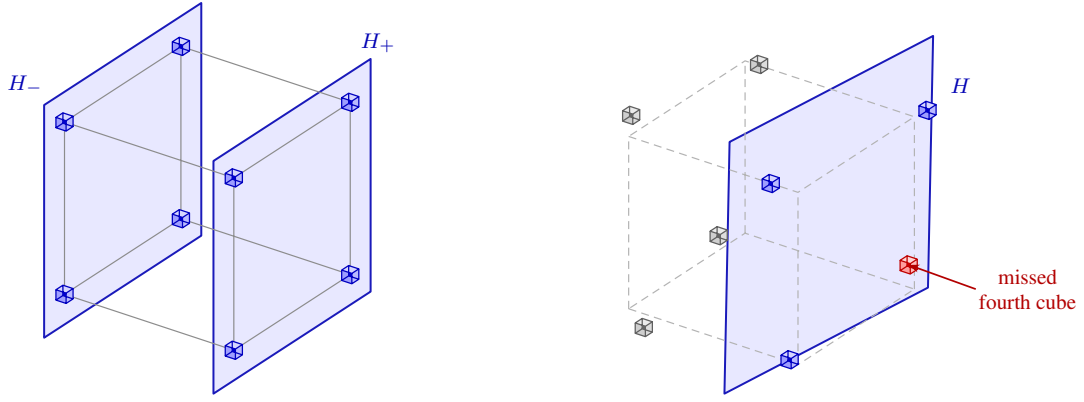
\begin{figure}[H]
  \centering
  \begin{subfigure}[t]{0.48\textwidth}
    \centering
    \ifshowgraphs
    \begin{tikzpicture}[
      x={(1.12cm,-0.37cm)},
      y={(0.77cm,0.50cm)},
      z={(0cm,1.14cm)},
      line join=round,
      font=\scriptsize
    ]
      \path[fill=blue!45,fill opacity=0.20,
            draw=blue!70!black,draw opacity=0.90,line width=0.75pt]
        (-1,-1.35,-1.35)--(-1,1.35,-1.35)--
        (-1,1.35,1.35)--(-1,-1.35,1.35)--cycle;
      \path[fill=blue!45,fill opacity=0.20,
            draw=blue!70!black,draw opacity=0.90,line width=0.75pt]
        (1,-1.35,-1.35)--(1,1.35,-1.35)--
        (1,1.35,1.35)--(1,-1.35,1.35)--cycle;

      \drawreferencecube[black!45]
      \foreach \yy in {-1,1}{
        \foreach \zz in {-1,1}{
          \drawvertexcube[blue]{-1}{\yy}{\zz}
          \drawvertexcube[blue]{1}{\yy}{\zz}
        }
      }

      \node[text=blue!70!black,fill=white,fill opacity=0.88,
            text opacity=1,inner sep=1pt,anchor=south east]
        at (-1,-1.32,1.42) {$H_-$};
      \node[text=blue!70!black,fill=white,fill opacity=0.88,
            text opacity=1,inner sep=1pt]
        at (1,1.48,1.48) {$H_+$};
    \end{tikzpicture}
    \fi
    \caption{Each of the hyperplanes $H_-$ and $H_+$ intersects four cubes.}
    \label{fig:boolean-cube-unperturbed}
  \end{subfigure}
  \hfill
  \begin{subfigure}[t]{0.48\textwidth}
    \centering
    \ifshowgraphs
    \begin{tikzpicture}[
      x={(1.12cm,-0.37cm)},
      y={(0.77cm,0.50cm)},
      z={(0cm,1.14cm)},
      line join=round,
      font=\scriptsize
    ]

      \path[fill=blue!45,fill opacity=0.20,
            draw=blue!70!black,draw opacity=0.90,line width=0.75pt]
        (0.5713060,-1.6399671,-1.1942291)--
        (0.7391282,-1.7959321,1.8494207)--
        (1.2271966,0.9910877,2.0159941)--
        (1.0593744,1.1470527,-1.0276557)--cycle;
      \drawreferencecube[gray!60,densely dashed]

      \drawvertexcube[gray]{-0.7352}{-1.1252}{-1.0767}
      \drawvertexcube[gray]{-1.0150}{-0.9352}{1.2112}
      \drawvertexcube[gray]{-1.2779}{0.9508}{-1.0983}
      \drawvertexcube[gray]{-1.0183}{1.2675}{0.8346}

      \drawvertexcube[blue]{0.7899}{-0.8386}{-1.0815}
      \drawvertexcube[blue]{0.7361}{-1.0841}{1.0538}
      \drawvertexcube[red]{0.9346}{0.9968}{-0.7389}
      \drawvertexcube[blue]{1.2315}{0.8887}{1.1982}

      \node[text=red!70!black,fill=white,fill opacity=0.90,
            text opacity=1,inner sep=1pt,align=center,anchor=west]
        (missed-label) at (2.00,0.60,-0.50) {missed\\fourth cube};
      \draw[-{Stealth[length=4pt]},red!70!black,line width=0.65pt]
        (missed-label.west) -- (0.9346,0.9968,-0.7389);
      \node[text=blue!70!black,fill=white,fill opacity=0.88,
            text opacity=1,inner sep=1pt]
        at (1.42,1.18,1.42) {$H$};
    \end{tikzpicture}
    \fi
    \caption{The closest plane intersecting the three blue cubes still misses the
    fourth.}
    \label{fig:boolean-cube-perturbed}
  \end{subfigure}

  \caption{The role of the perturbation in dimension $3$. The unperturbed centers are indexed by $\{-\frac12,\frac12\}^3$, the vertex neighborhoods
  are enlarged, and the perturbation is exaggerated so that it remains
  visible at this scale. Before perturbation, the two planes $H_\pm$ intersect
  $2^{3-1}=4$ cubes.  In the displayed perturbed configuration, no plane can intersect more than $d=3$ cubes.  The blue plane attains this
  bound, the missed fourth cube is red, and cubes irrelevant to this
  plane are gray. Best viewed in color.}
  \label{fig:boolean-cube-perturbation}
\end{figure}

\paragraph{Step 2: Creating $2^d$ copies of a given hard function.}
Proposition~\ref{prop:one-layer-copier} builds on the preceding lemma to place
$2^d$ copies of a given hard function in small cubes centered at the perturbed
vertices.  For $r=2^{-\Ocal(d^3)}$, it constructs a depth-$2$ ReLU map $S$ of
width $\Ocal(d^4)$ satisfying
\[
  S(\bq_{\bs}+r\bu)=\bu
  \qquad
  \text{for all }\bs\in\left\{-\frac12,\frac12\right\}^d
  \text{ and }\bu\in[-1,1]^d.
\]
Consequently, for any function $h$,
\[
  (h\circ S)(\bq_{\bs}+r\bu)=h(\bu)
  \qquad
  \text{for all }\bs\in\left\{-\frac12,\frac12\right\}^d
  \text{ and }\bu\in[-1,1]^d.
\]
As $\bu$ ranges over $[-1,1]^d$, the point $\bq_{\bs}+r\bu$ ranges over the
cube centered at $\bq_{\bs}$, while $S$ recovers the local coordinate $\bu$
exactly.  Thus, $h\circ S$ contains one exact translated and rescaled copy of
$h$ in each of the $2^d$ cubes.  Figure~\ref{fig:one-layer-copier-illustration}
illustrates this identity in dimension $2$ for the CPWL base function
$h_0(\bu)\coloneqq\relu{1-|u_1|}$ used in the induction.

When $h$ is computed by a ReLU network, the affine output map of $S$ merges into the first affine map of the network computing $h$. Thus, the composition $h\circ S$ adds exactly one hidden layer, of width $\Ocal(d^4)$.

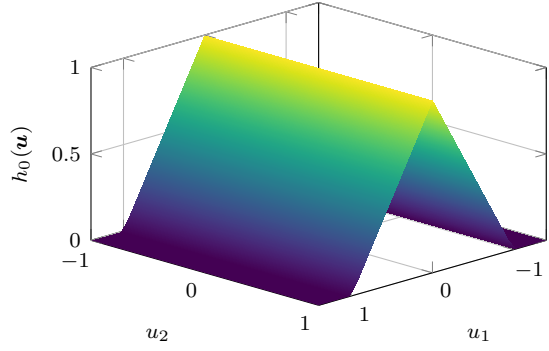
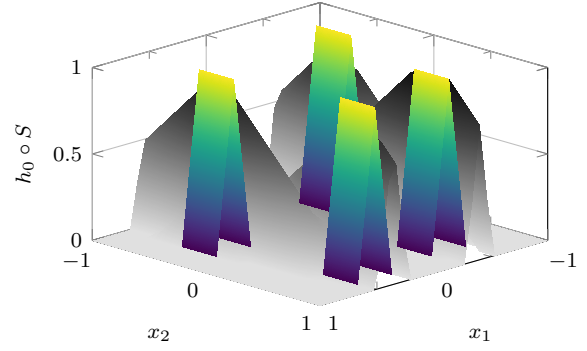
\begin{figure}[H]
  \centering
  \begin{subfigure}[t]{0.47\textwidth}
    \centering
    \ifshowgraphs
    \begin{tikzpicture}
      \begin{axis}[
        width=0.98\linewidth,
        height=0.72\linewidth,
        view={135}{28},
        xlabel={$u_1$},
        ylabel={$u_2$},
        zlabel={$h_0(\bu)$},
        label style={font=\scriptsize},
        tick label style={font=\scriptsize},
        domain=-1.4:1.4,
        y domain=-1:1,
        samples=31,
        samples y=9,
        zmin=0,
        zmax=1,
        xtick={-1,0,1},
        ytick={-1,0,1},
        ztick={0,0.5,1},
        colormap/viridis,
        grid=major
      ]
        \addplot3[surf,shader=interp] {max(0,1-abs(x))};
      \end{axis}
    \end{tikzpicture}
    \fi
    \caption{The CPWL base function $h_0(\bu)=\relu{1-|u_1|}$.}
    \label{fig:base-function}
  \end{subfigure}
  \hfill
  \begin{subfigure}[t]{0.47\textwidth}
    \centering
    \ifshowgraphs
    \begin{tikzpicture}
      \pgfplotsset{
        colormap={copiergray}{
          color(0cm)=(black!12);
          color(0.35cm)=(black!42);
          color(0.70cm)=(black!74);
          color(1cm)=(black!96)
        }
      }
      \begin{axis}[
        width=0.98\linewidth,
        height=0.72\linewidth,
        view={135}{28},
        xlabel={$x_1$},
        ylabel={$x_2$},
        zlabel={$h_0\circ S$},
        label style={font=\scriptsize},
        tick label style={font=\scriptsize},
        domain=-1:1,
        y domain=-1:1,
        samples=31,
        samples y=31,
        zmin=0,
        zmax=1,
        xtick={-1,0,1},
        ytick={-1,0,1},
        ztick={0,0.5,1},
        minor x tick num=1,
        minor y tick num=1,
        minor grid style={gray!20},
        colormap/viridis,
        grid=major,
        z buffer=sort,
        clip=false
      ]
        \addplot3[
          surf,
          shader=interp,
          colormap name=copiergray
        ]
        {
          max(
            max(
              max(0,1-abs((x+0.64)/0.15)
                -0.8*max(0,abs(y+0.52)-0.15)),
              max(0,1-abs((x+0.40)/0.15)
                -0.8*max(0,abs(y-0.58)-0.15))
            ),
            max(
              max(0,1-abs((x-0.55)/0.15)
                -0.8*max(0,abs(y+0.36)-0.15)),
              max(0,1-abs((x-0.36)/0.15)
                -0.8*max(0,abs(y-0.69)-0.15))
            )
          )
        };

        \addplot3[
          surf,shader=interp,colormap name=viridis,
          domain=-0.79:-0.49,y domain=-0.67:-0.37,
          samples=17,samples y=5
        ] {max(0,1-abs((x+0.64)/0.15))};
        \addplot3[
          surf,shader=interp,colormap name=viridis,
          domain=-0.55:-0.25,y domain=0.43:0.73,
          samples=17,samples y=5
        ] {max(0,1-abs((x+0.40)/0.15))};
        \addplot3[
          surf,shader=interp,colormap name=viridis,
          domain=0.40:0.70,y domain=-0.51:-0.21,
          samples=17,samples y=5
        ] {max(0,1-abs((x-0.55)/0.15))};
        \addplot3[
          surf,shader=interp,colormap name=viridis,
          domain=0.21:0.51,y domain=0.54:0.84,
          samples=17,samples y=5
        ] {max(0,1-abs((x-0.36)/0.15))};
      \end{axis}
    \end{tikzpicture}
    \fi
    \caption{Four perturbed copies of $h_0$, joined by a CPWL interpolation.}
    \label{fig:four-tent-copies}
  \end{subfigure}

  \caption{Creating $2^d$ copies of a hard function for $d=2$. The colored patches in
  Subfigure~\textup{(b)} are the four perturbed squares
  $\bq_{\bs}+r[-1,1]^2$. On each square, Proposition~\ref{prop:one-layer-copier} makes the surface an exact copy of
  Subfigure~\textup{(a)}; the grayscale
  surface shows the piecewise-linear interpolation outside the squares. The four
  centers have been perturbed from the grid $\{-\frac12,\frac12\}^2$. The copy size and the perturbation are exaggerated for visibility.  Best viewed in color.}
  \label{fig:one-layer-copier-illustration}
\end{figure}

\paragraph{Step 3: Iterating the construction.}
The induction begins with the base function $h_0(\bx)\coloneqq\relu{1-|x_1|}$ on
$(-1,1)^d$, shown in Subfigure~\ref{fig:base-function}.
Lemma~\ref{lem:base-gap} shows that its squared error under the uniform distribution is at least $1/12$ for any affine approximant.
After $j$ applications of the copying transformation $S$, the target contains copies of the preceding hard function, and the distribution assigns equal mass to them.
Figure~\ref{fig:copier-iteration} illustrates this recursion in
dimension $2$.  If a depth-$(j+1)$ network has width $w$, its first-layer
activation hyperplanes intersect at most $wd$ of the $2^d$ copies.  On every missed copy, the first layer has a fixed activation pattern and can be
merged into the next layer.  Applying the induction hypothesis on those
copies yields the accuracy lower bound
\[
  \frac1{12}\left(1-\frac{wd}{2^d}\right)^j.
\]
Taking $j=k-1$ and the width from
Theorem~\ref{thm:l2-hierarchy} makes this quantity at least $1/24$.  A final
input dilation makes the target globally $1$-Lipschitz without changing its
network complexity or squared approximation error; this dilation is also
what moves the support to exponential radius.

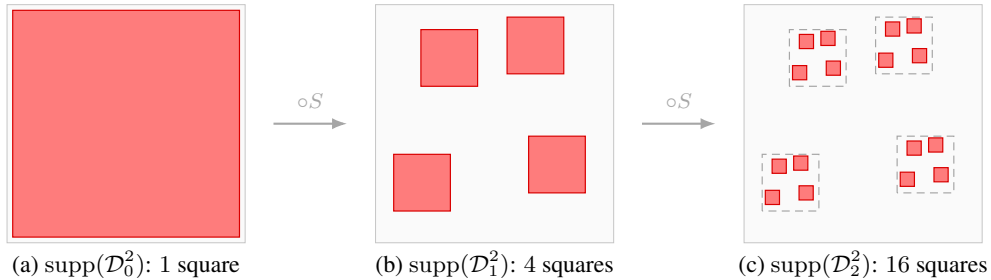
\begin{figure}[H]
  \centering
  \ifshowgraphs
  \begin{tikzpicture}[
    x=1.50cm,
    y=1.50cm,
    support/.style={
      fill=red!58,
      fill opacity=0.88,
      draw=red!85!black,
      line width=0.45pt
    },
    parent/.style={
      draw=gray!70,
      densely dashed,
      line width=0.45pt
    },
    background/.style={
      draw=gray!45,
      fill=gray!4,
      line width=0.35pt
    }
  ]

    \begin{scope}
      \draw[background] (-1.05,-1.05) rectangle (1.05,1.05);
      \draw[support] (-1,-1) rectangle (1,1);
      \node[align=center,font=\footnotesize] at (0,-1.24)
        {\textup{(a)} $\operatorname{supp}(\Dcal_0^2)$: $1$ square};
    \end{scope}

    \draw[-{Latex[length=2mm]},thick,gray!70]
      (1.30,0) -- (1.95,0)
      node[midway,above=2pt,font=\scriptsize] {$\circ S$};

    \begin{scope}[shift={(3.25,0)}]
      \draw[background] (-1.05,-1.05) rectangle (1.05,1.05);
      \foreach \qx/\qy in {
        -0.64/-0.52,
        -0.40/0.58,
         0.55/-0.36,
         0.36/0.69
      }{
        \pgfmathsetmacro{\leftx}{\qx-0.25}
        \pgfmathsetmacro{\rightx}{\qx+0.25}
        \pgfmathsetmacro{\bottomy}{\qy-0.25}
        \pgfmathsetmacro{\topy}{\qy+0.25}
        \draw[support]
          (\leftx,\bottomy) rectangle (\rightx,\topy);
      }
      \node[align=center,font=\footnotesize] at (0,-1.24)
        {\textup{(b)} $\operatorname{supp}(\Dcal_1^2)$: $4$ squares};
    \end{scope}

    \draw[-{Latex[length=2mm]},thick,gray!70]
      (4.55,0) -- (5.20,0)
      node[midway,above=2pt,font=\scriptsize] {$\circ S$};

    \begin{scope}[shift={(6.50,0)}]
      \draw[background] (-1.05,-1.05) rectangle (1.05,1.05);
      \foreach \qx/\qy in {
        -0.64/-0.52,
        -0.40/0.58,
         0.55/-0.36,
         0.36/0.69
      }{
        \pgfmathsetmacro{\leftx}{\qx-0.25}
        \pgfmathsetmacro{\rightx}{\qx+0.25}
        \pgfmathsetmacro{\bottomy}{\qy-0.25}
        \pgfmathsetmacro{\topy}{\qy+0.25}
        \draw[parent]
          (\leftx,\bottomy) rectangle (\rightx,\topy);
        \foreach \px/\py in {
          -0.64/-0.52,
          -0.40/0.58,
           0.55/-0.36,
           0.36/0.69
        }{
          \pgfmathsetmacro{\cx}{\qx+0.25*\px}
          \pgfmathsetmacro{\cy}{\qy+0.25*\py}
          \pgfmathsetmacro{\leftchild}{\cx-0.0625}
          \pgfmathsetmacro{\rightchild}{\cx+0.0625}
          \pgfmathsetmacro{\bottomchild}{\cy-0.0625}
          \pgfmathsetmacro{\topchild}{\cy+0.0625}
          \draw[support]
            (\leftchild,\bottomchild)
            rectangle
            (\rightchild,\topchild);
        }
      }
      \node[align=center,font=\footnotesize] at (0,-1.24)
        {\textup{(c)} $\operatorname{supp}(\Dcal_2^2)$: $16$ squares};
    \end{scope}
  \end{tikzpicture}
  \fi
  \caption{Iterating the distribution in dimension $2$.  Red marks the support, on which the distribution is uniform.  Each application replaces every square by four shrunken,
  perturbed copies, so one square becomes four and then sixteen.  The
  dashed boxes in Subfigure~\textup{(c)} show the parents from
  Subfigure~\textup{(b)}.  Sizes and perturbations are exaggerated for visibility.
  Best viewed in color.}
  \label{fig:copier-iteration}
\end{figure}

\section{\texorpdfstring{An Exact Separation Between Depths $4$ and $3$}
{An Exact Separation Between Depths 4 and 3}}
\label{sec:exact-separation}

The hierarchy theorem already yields exact separations as immediate
corollaries, but these inherit the exponential geometric scale of its
construction.  We now show that, at adjacent depths $4$ and $3$, an
exponential exact-computation separation can hold on the unit hypercube for a more regular target.

\begin{theorem}
\label{thm:exact-separation}
There is a universal constant $C>0$ such that, for all $d\geq3$,
there exists a globally defined CPWL function
\[
  g_d:\R^d\to[0,1]
\]
such that
\[
  \Lip(g_d)\leq2\sqrt d,
  \qquad
  g_d([-1,1]^d)=[0,1],
\]
and the following hold.
\begin{enumerate}
\item The function $g_d$ is computed globally by a depth-$4$
ReLU network whose hidden-layer widths are at most $Cd^4$, $2d$,
and $1$.

\item Any depth-$3$ ReLU network agreeing with $g_d$
throughout $(-1,1)^d$ has first hidden-layer width at
least
\[
  \frac{2^{d-1}}d.
\]

\end{enumerate}
\end{theorem}

To compare the geometric scales of these exact-computation separations, let $f$ denote the hard function, let $\Omega\subseteq\R^d$ be a bounded computation domain, let
\[
  R\coloneqq\sup_{\bx\in\Omega}\norm{\bx}_2,
\]
let $L$ denote the Lipschitz constant of the target on $\Omega$, and define
its image length by
\[
  I\coloneqq
  \sup_{\bx\in\Omega}f(\bx)
  -
  \inf_{\bx\in\Omega}f(\bx)>0.
\]
The quantity
\[
  \frac{RL}{I}
\]
is invariant under input and output rescaling, and these rescalings preserve
exact representability by networks of the same depth and width.  Indeed, for
$a>0$, replacing $f$ on $\Omega$ by $\bx\mapsto f(\bx/a)$ on the scaled
domain $a\Omega$ multiplies $R$ by $a$ and divides $L$ by $a$; the
corresponding network is obtained by dividing its first-layer weights by
$a$.  Scaling the target by a factor $b\neq0$ multiplies both $L$ and $I$ by
$|b|$; the corresponding network is obtained by scaling its output affine
map.  This is the exact-computation counterpart of the approximation
equivalence in \citet[Theorem~9]{safran-eldan-shamir-2019}.

For Theorem~\ref{thm:exact-separation}, we take
$\Omega=[-1,1]^d$.  This domain lies in a Euclidean ball of radius
$R=\sqrt d$, the target is at most $2\sqrt d$-Lipschitz there, and its
image is $[0,1]$.  Hence
\[
  \frac{RL}{I}
  \leq
  \frac{\sqrt d\,(2\sqrt d)}{1}
  =
  2d.
\]
By contrast, consider the special case obtained by choosing
$k=3$ in
Theorem~\ref{thm:l2-hierarchy}, and requiring equality on the support of
$\Dcal_3^d$. The construction has
\[
  R=2^{\Theta(d^3)},
  \qquad
  L=\Theta(1),
  \qquad
  I=1.
\]
Here the upper bound on $L$ follows from the global $1$-Lipschitz property,
while the matching constant lower bound follows because every cube contains an exact copy of the base CPWL function at a constant spatial
scale after the final dilation.  Consequently,
\[
  \frac{RL}{I}=2^{\Theta(d^3)}.
\]
Because this quantity is scale invariant, rescaling can move the exponential
factor among the domain radius, Lipschitz constant, and image length, but
cannot eliminate it.

The same observation applies to the iterated-sawtooth construction of
\citet{telgarsky-2016}.  The
$k$-fold composition of its basic triangle map has
image $[0,1]$ and Lipschitz constant
$2^{k}$ on the unit interval, and therefore has
normalization parameter
$\Theta(2^{k})$.  Dilating the input to
make the Lipschitz constant bounded makes the domain radius exponential,
whereas scaling the output instead shrinks the image length exponentially.

Thus, to the best of our knowledge,
Theorem~\ref{thm:exact-separation} gives the first exponential width lower
bound for exact computation in a regime where the domain radius and
Lipschitz constant are polynomially bounded and the target image has constant
length.  We emphasize that this comparison concerns geometric normalization only, and that the exact-computation lower bound in Theorem~\ref{thm:exact-separation} does not by itself imply hardness of approximation.

\subsection{Construction and proof technique}

The lower-bound mechanism has two ingredients. As shown by \citet{mukherjee2017lower} and \citet{Safran2026}, no depth-$2$ ReLU network agrees with the bivariate target
$(x,y)\mapsto\max\{0,x,y\}$ on any neighborhood of the origin. More
generally, we call a point a \emph{depth-$2$ hard point} of a target if no
depth-$2$ network agrees with the target on any neighborhood of that point.
If a depth-$3$ network computes the target, then at least one first-layer
activation hyperplane must pass through each of its depth-$2$ hard points.
Otherwise, all first-layer neurons have fixed activation states on some
neighborhood of the point, so the first hidden layer computes an affine transformation there and can be absorbed in the next layer, yielding a local depth-$2$ representation and a contradiction.

The second ingredient is a target with exponentially many depth-$2$ hard
points in general position. Any hyperplane in $\R^d$ contains
at most $d$ points from such a set. Consequently, covering all the hard
points requires exponentially many first-layer activation hyperplanes, and
hence exponentially many first-layer neurons.

The construction begins with the simple CPWL function
\[
    t_d(\bx)\coloneqq\relu{1-\norm{\bx}_1}.
\]
We use the $2^{d-1}$ points
\[
    \bp_{\bs}
    \coloneqq
    \left(\frac{2\bs}{d-1},0\right),
    \qquad
    \bs\in\left\{-\frac12,\frac12\right\}^{d-1}.
\]
Thus, the first $d-1$ coordinates of each $\bp_{\bs}$ independently take
the values $\pm 1/(d-1)$, while its last coordinate is zero. Each point lies
on the boundary of the unit $\ell_1$ ball, and a suitable two-dimensional
restriction of $t_d$ around it reduces local depth-$2$ representability to
the bivariate obstruction above. Hence, all the points $\bp_{\bs}$ are
depth-$2$ hard points.

We then use Lemma~\ref{lem:general-position-points} to construct perturbed points $\bq_{\bs}$ in general position, and
a CPWL map $U:\R^d\to\R^d$ that is uniformly close to the identity, and acts
as a translation from a neighborhood of each $\bq_{\bs}$ onto a neighborhood
of the corresponding $\bp_{\bs}$. The separating target is
\[
g_d\coloneqq t_d\circ U.
\]
Figure~\ref{fig:exact-perturbation} illustrates this mechanism in dimension
$2$. The hard points are displaced, while $U$ preserves the local function
around each perturbed point up to a translation of coordinates. Thus, the
local structure of $g_d$ around all perturbed points $\bq_{\bs}$ is identical to
that of $t_d$ around the corresponding depth-$2$ hard point $\bp_{\bs}$.

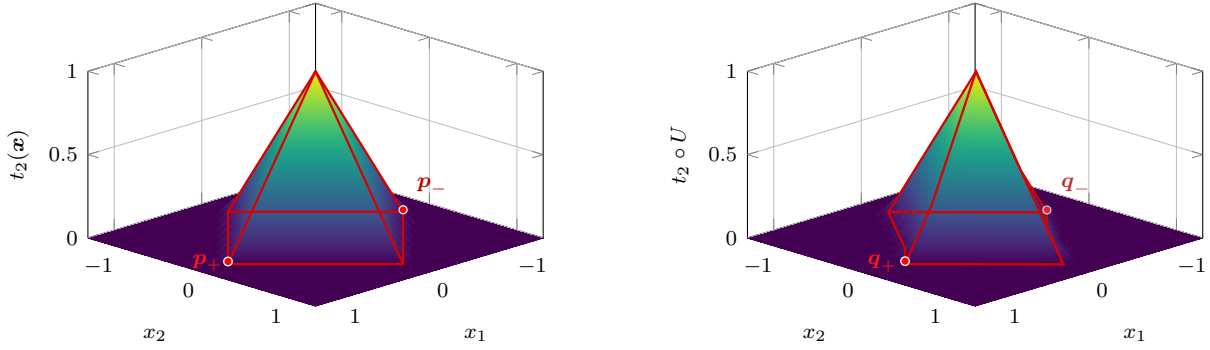
\begin{figure}[H]
  \centering
  \begin{subfigure}[t]{0.47\textwidth}
    \centering
    \ifshowgraphs
    \begin{tikzpicture}
      \begin{axis}[
        width=0.98\linewidth,
        height=0.72\linewidth,
        view={135}{30},
        xlabel={$x_1$},
        ylabel={$x_2$},
        zlabel={$t_2(\bx)$},
        label style={font=\scriptsize},
        tick label style={font=\scriptsize},
        domain=-1.3:1.3,
        y domain=-1.3:1.3,
        samples=19,
        samples y=19,
        zmin=0,
        zmax=1,
        xtick={-1,0,1},
        ytick={-1,0,1},
        ztick={0,0.5,1},
        colormap/viridis,
        grid=major,
        z buffer=sort,
        clip=false
      ]
        \addplot3[
          surf,shader=interp,colormap name=viridis
        ] {max(0,1-abs(x)-abs(y))};

        \draw[red!85!black,line width=0.85pt]
          (axis cs:-1,0,0)--(axis cs:0,1,0)--
          (axis cs:1,0,0)--(axis cs:0,-1,0)--cycle;
        \draw[red!85!black,line width=0.85pt]
          (axis cs:-1,0,0)--(axis cs:0,0,1)--
          (axis cs:1,0,0);
        \draw[red!85!black,line width=0.85pt]
          (axis cs:0,-1,0)--(axis cs:0,0,1)--
          (axis cs:0,1,0);

        \addplot3[
          only marks,
          mark=*,
          mark size=1.5pt,
          mark options={fill=red,draw=white,line width=0.45pt}
        ]
          coordinates {(-1,0,0.012)};
        \addplot3[
          only marks,
          mark=*,
          mark size=1.6pt,
          mark options={fill=red,draw=white,line width=0.50pt}
        ]
          coordinates {(1,0,0.016)};
        \node[font=\scriptsize,text=red!75!black,anchor=south west,
          xshift=1.5pt,yshift=1.5pt]
          at (axis cs:-1,0,0.015) {$\bp_-$};
        \node[font=\scriptsize\bfseries,text=red!90!white,anchor=center]
          at (axis cs:1.15,-0.08,0.020) {$\bp_+$};
      \end{axis}
    \end{tikzpicture}
    \fi
    \caption{The base function $t_2$.  Red curves outline its
    non-differentiability set, and its two depth-$2$ hard points are marked in
    red.}
    \label{fig:exact-base-function}
  \end{subfigure}
  \hfill
  \begin{subfigure}[t]{0.47\textwidth}
    \centering
    \ifshowgraphs
    \begin{tikzpicture}
      \begin{axis}[
        width=0.98\linewidth,
        height=0.72\linewidth,
        view={135}{30},
        xlabel={$x_1$},
        ylabel={$x_2$},
        zlabel={$t_2\circ U$},
        label style={font=\scriptsize},
        tick label style={font=\scriptsize},
        domain=-1.3:1.3,
        y domain=-1.3:1.3,
        samples=19,
        samples y=19,
        zmin=0,
        zmax=1,
        xtick={-1,0,1},
        ytick={-1,0,1},
        ztick={0,0.5,1},
        colormap/viridis,
        grid=major,
        z buffer=sort,
        clip=false,
        point meta min=0,
        point meta max=1,
        declare function={
          illustrationdx(\x)=ifthenelse(\x < -0.6,0.10,
            ifthenelse(\x > 0.6,-0.10,-\x/6));
          illustrationdy(\x)=ifthenelse(\x < -0.6,-0.08,
            ifthenelse(\x > 0.6,0.10,0.01+0.15*\x));
          perturbedheight(\x,\y)=max(0,1
            -abs(\x-illustrationdx(\x))
            -abs(\y-illustrationdy(\x)));
        }
      ]
        \addplot3[
          surf,shader=interp,colormap name=viridis
        ] {perturbedheight(x,y)};

        \draw[red!85!black,line width=0.85pt]
          (axis cs:-0.9,-0.08,0)--(axis cs:-0.6,0.22,0)--
          (axis cs:0,1.01,0)--(axis cs:0.6,0.40,0)--
          (axis cs:0.9,0.10,0);
        \draw[red!85!black,line width=0.85pt]
          (axis cs:-0.9,-0.08,0)--(axis cs:-0.6,-0.38,0)--
          (axis cs:0,-0.99,0)--(axis cs:0.6,-0.20,0)--
          (axis cs:0.9,0.10,0);

        \draw[red!85!black,line width=0.85pt]
          (axis cs:-0.9,-0.08,0)--(axis cs:-0.6,-0.08,0.30)--
          (axis cs:0,0.01,1)--(axis cs:0.6,0.10,0.30)--
          (axis cs:0.9,0.10,0);

        \draw[red!85!black,line width=0.85pt]
          (axis cs:0,-0.99,0)--(axis cs:0,0.01,1)--
          (axis cs:0,1.01,0);
        \addplot3[
          only marks,
          mark=*,
          mark size=1.6pt,
          mark options={fill=red,draw=white,line width=0.50pt}
        ]
          coordinates {(0.9,0.1,0.018)};
        \node[font=\scriptsize\bfseries,text=red!90!white,anchor=center]
          at (axis cs:1.08,0.02,0.020) {$\bq_+$};

        \addplot3[
          only marks,
          mark=*,
          mark size=1.5pt,
          opacity=0.78,
          mark options={fill=red!80,draw=white,line width=0.45pt}
        ]
          coordinates {(-0.9,-0.08,0.014)};
        \node[font=\scriptsize,text=red!70!black,opacity=0.82,
          anchor=south west,xshift=1.5pt,yshift=1.5pt]
          at (axis cs:-0.9,-0.08,0.018) {$\bq_-$};
      \end{axis}
    \end{tikzpicture}
    \fi
    \caption{The perturbed target $t_2\circ U$, with its transported
    non-differentiability set outlined. The perturbed points $\bq_+$ and $\bq_-$ are both marked in red.}
    \label{fig:exact-perturbed-target}
  \end{subfigure}

  \caption{A two-dimensional illustration of the exact construction.
  Subfigure~\textup{(a)} shows the base function with
  both hard points marked.  The red curves outline its non-differentiability
  set: the boundary $|x_1|+|x_2|=1$ together with the two interior ridges.
  Subfigure~\textup{(b)} traces the transported boundary and interior ridges in
  red.  Around each marked point,
  the perturbation acts as a translation, so the function's local behavior is
  unchanged.  Any additional minor creases created by the illustrative CPWL
  interpolation for $U$ are omitted from Subfigure~\textup{(b)}.  The displacement
  is enlarged for visibility.  Best viewed in color.}
  \label{fig:exact-perturbation}
\end{figure}

The perturbation moves each hard point while acting as a translation on a neighborhood of that point. In
dimension $d$, the $2^{d-1}$ perturbed hard points $\bq_{\bs}$ are in general position. The first ingredient forces a first-layer activation
hyperplane through each of them, while the second ensures that any such
hyperplane contains at most $d$ of them. This gives the lower bound
$2^{d-1}/d$. For the upper bound, one hidden layer computes $U$, after which
two further hidden layers compute $t_d$; together, they form the claimed
depth-$4$ network.

\section{Conclusions and Future Work}
\label{sec:conclusions}

We established a complete exponential hierarchy for ReLU networks across all
adjacent depths. For all $k\geq2$, our
construction gives a globally bounded and Lipschitz target that is computed
by a polynomial-width depth-$(k+1)$ network,
but requires exponential width to approximate to constant error at
depth $k$, even when the shallower network has
unrestricted weights. We also obtained a complementary exact-computation
separation between depths $4$ and $3$ on the unit hypercube, while keeping the
target globally bounded and its Lipschitz constant polynomially controlled.

At the level of architectural resources, the hierarchy is complete: it
separates all pairs of adjacent depths, and each additional layer yields an
exponential reduction in width.  The main direction for strengthening the
result is therefore to improve the regularity of the separation rather than
its depth or width dependence.  For our construction, no input rescaling
makes both the distribution's support radius and the target's Lipschitz
constant polynomially bounded.  We use the $1$-Lipschitz normalization, under
which the distribution is supported at exponential radius.  The construction
therefore falls outside the regularity regime of \citet{vardi-shamir-2020}.  Whether a comparable hierarchy can be established within that regime remains a major open problem.

A particularly natural first step is an exponential $L_2$ separation of
depth $4$ from depth $3$ satisfying the Vardi--Shamir assumptions. To the
best of our knowledge, such a result would not by itself imply any major new
threshold-circuit lower bound, unlike analogous separations in which the
shallower ReLU network has depth at least $4$. Separating depth $4$ from depth $3$ may therefore be the most amenable setting in which to determine whether the exponential support radius used here is essential or can be replaced by a polynomially bounded, or even constant-radius, approximation domain.  Resolving this case would provide the first exponential fixed-depth separation beyond depth $3$ versus depth $2$ within the full regularity regime, and would help clarify how far an adjacent-depth hierarchy can
be pushed before the known circuit-complexity barriers become unavoidable.

\subsection*{Acknowledgments}
This research is supported by Israel Science Foundation Grant 1753/25.

\bibliographystyle{plainnat}
\bibliography{citations}

\appendix

\section{Perturbing hypercube vertices into general position}

We index the Boolean-cube vertices directly by
$\left\{-\frac12,\frac12\right\}^d$; the sign of coordinate $s_i$ is therefore
$2s_i$. The following lemma slightly perturbs the hypercube vertices using a
polynomial-width map and gives a uniform lower bound on the corresponding
determinants.

\begin{lemma}[Perturbation of hypercube vertices]
\label{lem:general-position-points}
There is a universal constant $C>0$ with the following property.  Let
$d\geq2$.  There are points $\bq_{\bs}\in\R^d$, indexed by
$\bs\in\left\{-\frac12,\frac12\right\}^d$, and a globally
defined CPWL map $\Delta:\R^d\to\R^d$ such that:
\begin{enumerate}
  \item\label{item:perturbation-realization}
  $\Delta$ is computed by a vector-valued depth-$2$ ReLU
  network of width at most $Cd^4$, and
  \[
    \bq_{\bs}
    \coloneqq
    \bs+\Delta(\bs);
  \]

  \item\label{item:perturbation-bounds}
  globally,
  \[
    \sup_{\bx\in\R^d}\norm{\Delta(\bx)}_\infty
    \leq\frac1{64d},
    \qquad
    \Lip(\Delta)\leq\frac1{16};
  \]

  \item\label{item:perturbation-local-constancy}
  the displacement is locally constant at all perturbed
  vertices:
  \[
    \Delta(\bq_{\bs}+\bu)
    =
    \Delta(\bs)
    \qquad
    \text{whenever }\norm{\bu}_\infty\leq\frac1{8d};
  \]

  \item\label{item:perturbation-determinant}
  for all $d+1$ distinct indices
  $\bs_0,\ldots,\bs_d\in\left\{-\frac12,\frac12\right\}^d$, let
  $Q(\bs_0,\ldots,\bs_d)\in\R^{d\times d}$ be the matrix whose
  $i^\mathrm{th}$ column is $\bq_{\bs_i}-\bq_{\bs_0}$.  Then
  \[
    \left|\det Q(\bs_0,\ldots,\bs_d)\right|
    \geq
    2^{-Cd^3}.
  \]
\end{enumerate}
In particular, the points
$\left\{\bq_{\bs}:\bs\in\left\{-\frac12,\frac12\right\}^d\right\}$
are in general position.
\end{lemma}

\begin{proof}
We first construct a polynomial-size collection of sign vectors such
that every $d+1$ of them are linearly independent.  For
$A\subseteq[d]$ and $\bs\in\left\{-\frac12,\frac12\right\}^d$, define
\[
  \chi_A(\bs)\coloneqq\prod_{i\in A}(2s_i).
\]
Independently sample each of $A_1,\ldots,A_M$ uniformly from the power
set of $[d]$.  For $\bs\in\left\{-\frac12,\frac12\right\}^d$, define
\[
  \bz_{\bs}
  \coloneqq
  \left(
    \chi_{A_1}(\bs),\ldots,\chi_{A_M}(\bs)
  \right)\in\{-1,1\}^M.
\]
Fix $\bv\in\{-1,1\}^d\setminus\{\bone\}$ before drawing the subsets,
and choose a coordinate $i_0$ with $v_{i_0}=-1$.  For each $j$, set
$Y_j(\bv)\coloneqq\prod_{i\in A_j}v_i$.  Conditional on all membership choices
for $A_j$ except whether $i_0\in A_j$, including $i_0$ flips the sign
of $Y_j(\bv)$.  A uniformly random subset includes $i_0$ with
probability $\frac12$, independently of all other membership choices.
Thus $Y_j(\bv)$ is an unbiased sign, and the variables
$Y_1(\bv),\ldots,Y_M(\bv)$ are independent because the subsets are
drawn independently.  For each fixed $\bv$, Hoeffding's inequality
bounds each of the upper- and lower-tail events by
$\exp\left(-\frac{M}{8d^2}\right)$.  There are $2^d-1$ admissible vectors
$\bv\in\{-1,1\}^d\setminus\{\bone\}$.  A union bound over both tails
and all such vectors therefore gives
\[
  \Prob_{A_1,\ldots,A_M}\left[
    \max_{\bv\in\{-1,1\}^d\setminus\{\bone\}}
    \left|
      \frac1M\sum_{j=1}^M Y_j(\bv)
    \right|>\frac1{2d}
  \right]
  \leq
  2^{d+1}\exp\left(-\frac{M}{8d^2}\right).
\]
Take $M\coloneqq\lceil C_1d^3\rceil$ for a sufficiently large
universal constant $C_1$.  The displayed probability bound is then
strictly below one.  By the probabilistic method, there exists a
realization of $A_1,\ldots,A_M$ such that
\begin{equation}
  \label{eq:small-feature-correlations}
    \max_{\bv\in\{-1,1\}^d\setminus\{\bone\}}
    \left|
      \frac1M\sum_{j=1}^M Y_j(\bv)
    \right|\le\frac1{2d}.
\end{equation}
For any two distinct
$\bs,\bs'\in\left\{-\frac12,\frac12\right\}^d$, define
$\bv\in\{-1,1\}^d$ by $v_i\coloneqq4s_is_i'$.  Then $\bv\neq\bone$ since $\bs,\bs'$ are distinct, and
\[
    \chi_{A_j}(\bs)\chi_{A_j}(\bs') = \prod_{i\in A_j}(2s_i) \prod_{i\in A_j}(2s'_i) = \prod_{i\in A_j}4s_is'_i = Y_j(\bv)
\]
for all $j$.  Therefore,
\[
  \frac1M\ip{\bz_{\bs}}{\bz_{\bs'}}
  =
  \frac1M\sum_{j=1}^M Y_j(\bv).
\]
Thus, by Equation~\eqref{eq:small-feature-correlations}, distinct
normalized vectors $M^{-\frac12}\bz_{\bs}$ have inner
product of magnitude at most $\frac1{2d}$.  The Gram matrix of any
$m\leq d+1$ of them is strictly diagonally
dominant, since
\[
  \frac{m-1}{2d}
  \le \frac{d}{2d}=\frac12.
\]
Hence every $d+1$ distinct vectors among the $\bz_{\bs}$ are linearly
independent.

We next realize these signs by a CPWL map that is constant near every
sign vertex.  Choose a continuous piecewise-linear function
$g:\R\to[-1,1]$ which equals $(-1)^j$ on
\[
  [j-\tfrac14,j+\tfrac14],
  \qquad j=0,\ldots,d,
\]
interpolates linearly between consecutive plateaus, and is constant on
its two unbounded tails.

\begin{figure}[ht]
  \centering
  \ifshowgraphs
  \begin{tikzpicture}[x=1.55cm,y=0.65cm]
    \draw[->] (-0.65,0) -- (3.65,0) node[right] {$t$};
    \draw[->] (0,-1.55) -- (0,1.55) node[above] {$g(t)$};
    \draw[densely dashed,gray] (-0.25,-1.25) -- (-0.25,1.25);
    \draw[densely dashed,gray] (0.25,-1.25) -- (0.25,1.25);
    \draw[densely dashed,gray] (0.75,-1.25) -- (0.75,1.25);
    \draw[densely dashed,gray] (1.25,-1.25) -- (1.25,1.25);
    \draw[densely dashed,gray] (1.75,-1.25) -- (1.75,1.25);
    \draw[densely dashed,gray] (2.25,-1.25) -- (2.25,1.25);
    \draw[densely dashed,gray] (2.75,-1.25) -- (2.75,1.25);
    \draw[densely dashed,gray] (3.25,-1.25) -- (3.25,1.25);
    \draw[very thick,blue]
      (-0.55,1) -- (0.25,1) -- (0.75,-1) -- (1.25,-1)
      -- (1.75,1) -- (2.25,1) -- (2.75,-1) -- (3.55,-1);
    \foreach \j in {0,1,2,3} {
      \draw (\j,0.08) -- (\j,-0.08) node[below=2pt] {$\j$};
    }
    \node[left] at (0,1) {$1$};
    \node[left] at (0,-1) {$-1$};
  \end{tikzpicture}
  \fi
  \caption{The plateau function $g$, shown for $d=3$.  Around each
  integer $j\in\{0,\ldots,d\}$, it is constant and equal to $(-1)^j$.
  Hence a small perturbation of the input to $g$ preserves the parity
  encoded at that integer; this is
  what makes the feature map below locally constant around all
  Boolean-cube vertices.}
  \label{fig:plateau-function}
\end{figure}
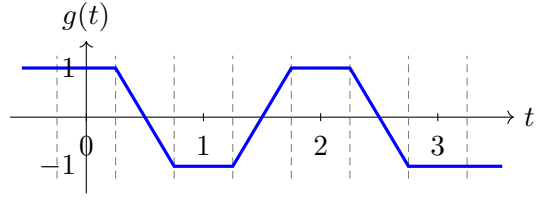

The function $g$ has exactly $2d$ breakpoints and can therefore be
represented by an affine combination of exactly $2d$ scalar ReLUs.
For $A\subseteq[d]$, consider the function
\[
  \bx\mapsto
  g\left(
    \sum_{i\in A}\left(\frac12-x_i\right)
  \right).
\]
At $\bx=\bs\in\left\{-\frac12,\frac12\right\}^d$, the sum in the argument is
$\sum_{i\in A}\left(\frac12-s_i\right)$, the number of negative coordinates of
$\bs$ in $A$.  Consequently,
\[
  g\left(\sum_{i\in A}\left(\frac12-s_i\right)\right)
  =
  (-1)^{\sum_{i\in A}\left(\frac12-s_i\right)}
  =
  \prod_{i\in A}(2s_i)
  =
  \chi_A(\bs).
\]
Define the vector-valued map $\widetilde z:\R^d\to\R^M$ directly
from these plateau functions by
\[
  \widetilde z(\bx)
  \coloneqq
  \left(
    g\left(
      \sum_{i\in A_j}\left(\frac12-x_i\right)
    \right)
  \right)_{j=1}^M.
\]
If
$\norm{\bx-\bs}_\infty<\frac1{4d}$, then, for all $A\subseteq[d]$, the two corresponding arguments of $g$ lie on the same plateau:
\[
  \left|
    \sum_{i\in A}\left(\frac12-x_i\right)
    -
    \sum_{i\in A}\left(\frac12-s_i\right)
  \right|
  \leq
  \sum_{i\in A}|x_i-s_i|
  <\frac14.
\]
Thus, if $\norm{\bx-\bs}_\infty<\frac1{4d}$, we have
\begin{equation}
  \label{eq:plateau-vector}
  \widetilde z(\bx)=\bz_{\bs}.
\end{equation}
Each coordinate of $\widetilde z$ is an affine combination of
exactly $2d$ scalar ReLUs.  Thus all $M$ coordinates are computed by one
depth-$2$ vector-valued ReLU network with at most
$2dM=\Ocal(dM)=\Ocal(d^4)$ hidden neurons and an affine output map. All coordinates are bounded by one.  Moreover, each coordinate is
$4\sqrt d$-Lipschitz, because $g$ has slopes of magnitude at most
$4$ and the affine map inside it has gradient norm at most
$\sqrt d$.  Summing the squared coordinatewise bounds gives
\begin{equation}
  \label{eq:feature-lipschitz}
  \Lip(\widetilde z)\leq4\sqrt{Md}.
\end{equation}

At this point, $\widetilde z$ maps the $d$-dimensional input into
$\R^M$ and is locally constant around every
$\bs\in\left\{-\frac12,\frac12\right\}^d$, with value $\bz_{\bs}$.
It remains to map these feature vectors back to $\R^d$ by a fixed linear map
whose norm is small and that places the perturbed points in general position.

Fix one
ordering of each $(d+1)$-subset, write it as
$I=(\bs_0,\ldots,\bs_d)$.  Define $X_I\in\R^{d\times d}$ to be
the matrix whose $i^\mathrm{th}$ column is $\bs_i-\bs_0$, and define
$Z_I\in\R^{M\times d}$ to be the matrix whose $i^\mathrm{th}$ column
is $\bz_{\bs_i}-\bz_{\bs_0}$.
The linear independence just proved implies that $Z_I$ has rank $d$.
Choose a left inverse $L_I\in\R^{d\times M}$, so that
\begin{equation}
  \label{eq:left-inverse}
  L_IZ_I=I_d.
\end{equation}

A well-chosen matrix $B\in\R^{d\times M}$ maps the plateau features back to
$\R^d$ and produces the candidate perturbations
$B\bz_{\bs}$.  We will choose one small rational matrix $B$ below.
For a fixed tuple $I$, the difference matrix of the resulting points
$\bs+B\bz_{\bs}$ is $X_I+BZ_I$.  We regard the $dM$ entries of $B$
as scalar variables.  For $i,j\in[d]$, the corresponding entry of
this difference matrix is
\[
  (X_I+BZ_I)_{ij}
  =
  (X_I)_{ij}+\sum_{\ell=1}^M
  B_{i\ell}(Z_I)_{\ell j},
\]
which is a polynomial in those variables.  Accordingly, define
\[
  D_I(B)\coloneqq\det(X_I+BZ_I).
\]
By the determinant expansion, $D_I$ is a polynomial in the entries of
$B$.  Its total degree is at most
$d$, because each term in the determinant expansion is a product of
$d$ entries of $X_I+BZ_I$.  This polynomial is not identically zero.
Indeed, substituting $B=(I_d-X_I)L_I$ and using
Equation~\eqref{eq:left-inverse} gives
\[
  X_I+BZ_I
  =
  X_I+(I_d-X_I)L_IZ_I
  =
  I_d,
\]
and hence $D_I((I_d-X_I)L_I)=1$.

Our goal is to choose one rational matrix $B$ for which
$D_I(B)\neq0$ simultaneously for all $(d+1)$-tuples $I$.
This ensures that no $d+1$ of the perturbed points
$\bs+B\bz_{\bs}$ lie on a common hyperplane.
Two additional quantitative requirements are important.
The entries of $B$ must be small enough that the resulting displacement
map satisfies the uniform and Lipschitz bounds asserted in the lemma,
while their common denominator must be controlled so that a non-zero
determinant cannot be arbitrarily small.  The latter will yield the
inverse-exponential lower bound on the determinant in the lemma.

We enforce all the non-vanishing requirements simultaneously by
considering the polynomial
\begin{equation}
  \label{eq:determinant-product}
  \Pi(B)\coloneqq\prod_I D_I(B).
\end{equation}
We then find a rational grid point where $\Pi$ is non-zero; the range
of the grid controls
the size of $B$, while its common denominator gives the quantitative
determinant lower bound.

There are $\binom{2^d}{d+1}$ possible choices of $I$.  Each determinant
$D_I$ is a non-zero polynomial of total degree at most $d$.
Therefore, $\Pi$ is a non-zero polynomial of total degree at most
\[
  d\binom{2^d}{d+1}
  \leq
  d(2^d)^{d+1}
  <
  2^d 2^{d(d+1)}
  =
  2^{d(d+2)}.
\]
Here the first inequality bounds the number of $(d+1)$-subsets by the
number of ordered $(d+1)$-tuples, and the strict inequality uses
$d<2^d$.  The polynomial $\Pi$ is non-zero because a finite product of non-zero
real polynomials is itself non-zero.

Let
\begin{equation}
  \label{eq:definition-L}
  L\coloneqq\lceil64dM\rceil.
\end{equation}
\begin{samepage}
A non-zero polynomial of total degree at most $D$ cannot vanish on a
Cartesian grid $S^n$ when $|S|>D$.  We prove this by induction on $n$.
For $n=1$, this is the elementary univariate root bound: a non-zero
polynomial of degree at most $D$ has at most $D$ distinct roots.  For
$n>1$, write the polynomial uniquely as
\[
  F(x_1,\ldots,x_n)
  =
  \sum_{\ell=0}^r
  c_{\ell}(x_1,\ldots,x_{n-1})
  x_n^{\ell}.
\]
Because $F$ is not the zero polynomial, its coefficient polynomials
$c_0,\ldots,c_r$ cannot all be zero.  Choose an index
$\ell$ for which
$c_\ell$ is non-zero.  The polynomial
$c_\ell$ need not depend on all of the first $n-1$ variables; it is simply a polynomial in
$x_1,\ldots,x_{n-1}$, and its total degree is at most $D$.
By the induction hypothesis, choose $(x_1,\ldots,x_{n-1})\in S^{n-1}$
at which $c_\ell$ is non-zero.  After this substitution,
the monomial $x_n^\ell$ has a non-zero coefficient, so
$x_n\mapsto F(x_1,\ldots,x_{n-1},x_n)$ is a non-zero univariate polynomial in
$x_n$ of degree at most $D$.  It vanishes on at most $D$ elements of
$S$, so some choice of $x_n\in S$ makes $F$ non-zero.

Apply this fact to the polynomial $\Pi$ defined in
Equation~\eqref{eq:determinant-product}, with a grid having
$2^{d(d+2)}+1$ values in each of its $dM$ coordinates.  It follows that
some matrix
\[
  B\in
  \left\{
    0,\frac1{L2^{d(d+2)}},\ldots,\frac1L
  \right\}^{d\times M}
\]
satisfies $\Pi(B)\neq0$.  If some factor $D_I(B)$ were zero, then
$\Pi(B)$ would be zero; hence all factors $D_I(B)$ are
non-zero.  Thus this single matrix $B$ makes the difference
matrix of every $(d+1)$-tuple non-singular, which is precisely the
simultaneous nonsingularity property
sought above.
\end{samepage}
Define
\[
  \Delta(\bx)\coloneqq B\widetilde z(\bx),
  \qquad
  \bq_{\bs}\coloneqq\bs+B\bz_{\bs} = \bs+\Delta(\bs).
\]
For all tuples $I$, the matrix whose columns are
$\bq_{\bs_i}-\bq_{\bs_0}$ is $X_I+BZ_I$ and therefore has determinant
$D_I(B)\neq0$. All entries of $B$ lie in $\left[0,\frac1L\right]$, while all coordinates of
$\widetilde z(\bx)$ have magnitude at most one.  Thus, for all
$i\in[d]$ and all $\bx\in\R^d$,
\[
  |\Delta_i(\bx)|
  =
  \left|\sum_{j=1}^M B_{ij}\widetilde{z}_j(\bx)\right|
  \leq
  \sum_{j=1}^M |B_{ij}|\,|\widetilde z_j(\bx)|
  \leq
  \sum_{j=1}^M\frac1L
  =
  \frac ML.
\]
Consequently,
\[
  \sup_{\bx}\norm{\Delta(\bx)}_\infty
  \leq\frac ML
  \leq\frac1{64d},
\]
where the final inequality follows directly from
Equation~\eqref{eq:definition-L}.
Moreover, the operator norm of a matrix is at most its Frobenius norm,
so
\[
  \norm{B}_{\mathrm{op}}
  \leq
  \norm{B}_{\mathrm F}
  =
  \sqrt{\sum_{i=1}^d\sum_{j=1}^M B_{ij}^2}
  \leq
  \frac{\sqrt{dM}}L.
\]
Combining this estimate with
Equation~\eqref{eq:feature-lipschitz} gives a global Lipschitz bound on
$\R^d$.  Indeed, for all $\bx_1,\bx_2\in\R^d$,
\[
  \begin{aligned}
    \norm{\Delta(\bx_1)-\Delta(\bx_2)}_2
    &\leq
    \norm{B}_{\mathrm{op}}
    \norm{\widetilde z(\bx_1)-\widetilde z(\bx_2)}_2\\
    &\leq
    \frac{\sqrt{dM}}L\,4\sqrt{Md}\,\norm{\bx_1-\bx_2}_2\\
    &\leq
    \frac{4dM}{L}\norm{\bx_1-\bx_2}_2
    \leq
    \frac1{16}\norm{\bx_1-\bx_2}_2.
  \end{aligned}
\]
The last inequality again uses Equation~\eqref{eq:definition-L}.
Therefore, $\Lip(\Delta)\leq\frac1{16}$ on all of $\R^d$.

Set $R\coloneqq2L2^{d(d+2)}$. All coordinates of the points $\bq_{\bs}$ are integers divided by $R$.  Since $M=\lceil C_1d^3\rceil=\Ocal(d^3)$,
Equation~\eqref{eq:definition-L} gives $L=\Ocal(d^4)$ and hence
\[
  R
  =
  2L2^{d(d+2)}
  \leq
  2^{d(d+2)+\Ocal(\log d)}
  =
  2^{\Ocal(d^2)}.
\]
Consequently, for all tuples $I$, the non-zero determinant $D_I(B)$
is an integer divided by $R^d$ and therefore has magnitude at least
\[
  R^{-d}
  \geq
  2^{-C_2d^3}
\]
for a sufficiently large universal constant $C_2>0$.  Enlarging the
universal constant in the lemma statement proves its determinant bound, which also gives general position.

Finally, if $\norm{\bu}_\infty\leq\frac1{8d}$, then
\[
  \begin{aligned}
    \norm{\bq_{\bs}+\bu-\bs}_\infty
    &=
    \norm{(\bq_{\bs}-\bs)+\bu}_\infty\\
    &\leq
    \norm{\bq_{\bs}-\bs}_\infty+\norm{\bu}_\infty\\
    &=
    \norm{\Delta(\bs)}_\infty+\norm{\bu}_\infty\\
    &\leq
    \frac1{64d}+\frac1{8d}
    <\frac1{4d}.
  \end{aligned}
\]
The plateau identity in Equation~\eqref{eq:plateau-vector} now gives
\[
  \Delta(\bq_{\bs}+\bu)
  =
  B\bz_{\bs}
  =
  \Delta(\bs),
\]
as claimed.
\end{proof}

\section{Proof of Theorem~\ref{thm:exact-separation}}

\begin{lemma}[Local depth-$2$ hardness of $t_d$]
\label{lem:local-td-hardness}
Let $d\geq3$, and define
\[
  t_d(\bx)\coloneqq\relu{1-\norm{\bx}_1}
  \qquad\text{and}\qquad
  \bp_{\bs}\coloneqq\left(\frac{2\bs}{d-1},0\right),
  \qquad \bs\in\left\{-\frac12,\frac12\right\}^{d-1}.
\]
No finite depth-$2$ ReLU network agrees with $t_d$ on a neighborhood
of any $\bp_{\bs}$.
\end{lemma}

\begin{proof}
Each $\bp_{\bs}$ lies on the boundary of the unit $\ell_1$ ball because
\[
  \norm{\bp_{\bs}}_1
  =
  \sum_{i=1}^{d-1}\frac{2|s_i|}{d-1}
  =1.
\]
Fix $\bs$.  Starting at $\bp_{\bs}$, consider moving in a direction
$(\bv,z)\in\R^{d-1}\times\R$ and evaluating $t_d$ at
$\bp_{\bs}+(\bv,z)$.  If
$\norm{\bv}_\infty<\frac1{d-1}$, then, for all $i\in[d-1]$,
\[
  \operatorname{sign}\left(\frac{2s_i}{d-1}+v_i\right)
  =
  \operatorname{sign}\left(\frac{2s_i}{d-1}\right)
  =
  \operatorname{sign}(s_i),
\]
because $\left|2s_i/(d-1)\right|=1/(d-1)>|v_i|$.  Hence
\[
  \left|\frac{2s_i}{d-1}+v_i\right|
  =
  \left|2s_i\right|\left|\frac{2s_i}{d-1}+v_i\right|
  =
  \left|\frac{4s_i^2}{d-1}+2s_iv_i\right|
  =
  \frac1{d-1}+2s_iv_i.
\]
It follows that
\begin{align*}
  \norm{\bp_{\bs}+(\bv,z)}_1
  &=
  \sum_{i=1}^{d-1}
  \left|\frac{2s_i}{d-1}+v_i\right|+|z|\\
  &=
  1+2\ip{\bs}{\bv}+|z|,
\end{align*}
and therefore, whenever $\norm{\bv}_\infty<\frac1{d-1}$,
\begin{equation}
  \label{eq:local-td-form}
  t_d\bigl(\bp_{\bs}+(\bv,z)\bigr)
  =
  \relu{-2\ip{\bs}{\bv}-|z|}.
\end{equation}

Now restrict $t_d$ to the two-dimensional plane through
$\bp_{\bs}$ parameterized by $t,z\in\R$ by setting
\[
  \bv\coloneqq-\frac{2t}{d-1}\bs.
\]
Since $\ip{\bs}{\bs}=\frac{d-1}{4}$, the restriction of
Equation~\eqref{eq:local-td-form} to this plane is $\relu{t-|z|}$.
Under the invertible linear
change of variables
\[
  t\coloneqq\frac{x+y}{2},
  \qquad
  z\coloneqq\frac{x-y}{2},
\]
we have
\[
  t-|z|
  =
  \frac{x+y-|x-y|}{2}
  =
  \min\{x,y\}.
\]
Thus a depth-$2$ representation of $t_d$ on a neighborhood of
$\bp_{\bs}$ would give a local depth-$2$ representation
for $\relu{\min\{x,y\}}$.  Adding the two ReLU neurons $\relu{x}$ and
$\relu{y}$ would then give a local depth-$2$ representation of
\[
  \relu{x}+\relu{y}-\relu{\min\{x,y\}}
  =
  \max\{0,x,y\}.
\]
This contradicts the two-dimensional $\max\{0,x,y\}$ obstruction in
the proof of Proposition~A.2 of \citet{Safran2026}.\footnote{Proposition~A.2
is stated as an impossibility result for computing the maximum of three inputs by a
depth-$2$ ReLU network.  Its proof establishes the stronger local fact
used here: no finite depth-$2$ ReLU network agrees with
$\max\{0,x,y\}$ on any neighborhood of the origin.}
This proves the lemma.
\end{proof}

\begin{proof}[Proof of Theorem~\ref{thm:exact-separation}]
Around each marked point in Lemma~\ref{lem:local-td-hardness}, the
function $t_d$ has a local form that no finite depth-$2$ network can
represent.
These points are not in general position, and therefore we perturb them while
preserving the function exactly on a neighborhood of every point.

Apply Lemma~\ref{lem:general-position-points} in dimension $d$.
Denote its displacement map by
$\widehat\Delta$ and its perturbed sign vertices by
\[
  \widehat{\bq}_{\boldsymbol{\sigma}},
  \qquad
  \boldsymbol{\sigma}\in\left\{-\frac12,\frac12\right\}^d,
\]
whose unperturbed locations are $\boldsymbol{\sigma}$.

We use only the face indexed by $\left(\bs,-\frac12\right)$, where
$\bs\in\left\{-\frac12,\frac12\right\}^{d-1}$.  Translation by
$\frac{\be_d}{2}$ sends this face to the hyperplane $x_d=0$, and
scaling by $\frac2{d-1}$ sends its
unperturbed vertices to the hard points:
\[
  \frac{2}{d-1}
  \left(\left(\bs,-\frac12\right)+\frac{\be_d}{2}\right)
  =
  \left(\frac{2\bs}{d-1},0\right)
  =\bp_{\bs}.
\]
Apply the same affine change of coordinates to the perturbed vertices
and express the displacement map in the new coordinates:
\begin{equation}
  \label{eq:rescaled-points-and-displacement}
  \begin{aligned}
    \bq_{\bs}
    &\coloneqq
    \frac{2}{d-1}\left(
      \widehat{\bq}_{\left(\bs,-\frac12\right)}+\frac{\be_d}{2}
    \right),\\
    \Delta(\bx)
    &\coloneqq
    \frac{2}{d-1}\widehat\Delta\left(
      \frac{d-1}{2}\bx-\frac{\be_d}{2}
    \right).
  \end{aligned}
\end{equation}
Item~\ref{item:perturbation-realization} of
Lemma~\ref{lem:general-position-points} and the same change
of coordinates also give
\begin{equation}
  \label{eq:perturbed-hard-points}
  \bq_{\bs}=\bp_{\bs}+\Delta(\bp_{\bs}).
\end{equation}

The points $\widehat{\bq}_{\left(\bs,-\frac12\right)}$ form a subset of the
general-position family supplied by the lemma.  A common
translation and non-zero scaling preserve general position, so
the $2^{d-1}$ points $\bq_{\bs}$ are also in general position.
The definition of $\Delta$ in
Equation~\eqref{eq:rescaled-points-and-displacement} preserves the
network depth and width and, together with
Item~\ref{item:perturbation-bounds} of
Lemma~\ref{lem:general-position-points}, gives
\begin{equation}
  \label{eq:rescaled-displacement-bounds}
  \sup_{\bx\in\R^d}\norm{\Delta(\bx)}_\infty
  \leq
  \frac1{32d(d-1)},
  \qquad
  \Lip(\Delta)\leq\frac1{16}.
\end{equation}
Item~\ref{item:perturbation-local-constancy} of
Lemma~\ref{lem:general-position-points}
preserves this local form exactly.  If
$\norm{\bu}_\infty\leq\frac1{4d(d-1)}$, then
\[
  \frac{d-1}{2}(\bq_{\bs}+\bu)-\frac{\be_d}{2}
  =
  \widehat{\bq}_{\left(\bs,-\frac12\right)}+\frac{d-1}{2}\bu,
  \qquad
  \norm{\frac{d-1}{2}\bu}_\infty\leq\frac1{8d}.
\]
Using the definition of $\Delta$ in
Equation~\eqref{eq:rescaled-points-and-displacement}, the above
displayed equation, and
Item~\ref{item:perturbation-local-constancy} of
Lemma~\ref{lem:general-position-points}, we obtain
\begin{equation}
  \label{eq:rescaled-local-constancy}
  \begin{aligned}
  \Delta(\bq_{\bs}+\bu)
  &=
  \frac{2}{d-1}\widehat\Delta\left(
    \widehat{\bq}_{\left(\bs,-\frac12\right)}+\frac{d-1}{2}\bu
  \right)\\
  &=
  \frac{2}{d-1}\widehat\Delta\left(
    \left(\bs,-\frac12\right)
  \right)\\
  &=
  \Delta(\bp_{\bs}).
  \end{aligned}
\end{equation}
The first and final equalities in
Equation~\eqref{eq:rescaled-local-constancy} use
Equation~\eqref{eq:rescaled-points-and-displacement}; the middle equality uses
Item~\ref{item:perturbation-local-constancy} of
Lemma~\ref{lem:general-position-points}.

Define
\begin{equation}
  \label{eq:constructed-map-and-target}
  U(\bx)\coloneqq\bx-\Delta(\bx),
  \qquad
  g_d\coloneqq t_d\circ U.
\end{equation}
Equations~\eqref{eq:perturbed-hard-points} and
\eqref{eq:rescaled-local-constancy}, together with the definition of
$U$ in Equation~\eqref{eq:constructed-map-and-target}, imply
\begin{equation}
  \label{eq:exact-transport-map}
  U(\bq_{\bs}+\bu)=\bp_{\bs}+\bu
  \qquad
  \left(\norm{\bu}_\infty\leq\frac1{4d(d-1)}\right).
\end{equation}
Consequently, Equations~\eqref{eq:constructed-map-and-target} and
\eqref{eq:exact-transport-map} give
\begin{equation}
  \label{eq:transported-local-function}
  g_d(\bq_{\bs}+\bu)
  =
  t_d(\bp_{\bs}+\bu)
\end{equation}
for all $\bu$ in this neighborhood of $\bzero$.  Thus the entire local
form of $t_d$ around $\bp_{\bs}$ is reproduced around $\bq_{\bs}$, so no
finite depth-$2$ ReLU network agrees with $g_d$ on a neighborhood of
$\bq_{\bs}$: otherwise, translating its input and using
Equation~\eqref{eq:transported-local-function} would contradict
Lemma~\ref{lem:local-td-hardness}.

All marked points lie in the open unit hypercube.  Indeed,
Equations~\eqref{eq:perturbed-hard-points} and
\eqref{eq:rescaled-displacement-bounds} give
\[
  \norm{\bq_{\bs}}_\infty
  \leq
  \norm{\bp_{\bs}}_\infty+\norm{\Delta(\bp_{\bs})}_\infty
  \leq
  \frac1{d-1}+\frac1{32d(d-1)}
  <1.
\]

Suppose now that a depth-$3$ ReLU network $N$ agrees with
$g_d$ throughout $(-1,1)^d$.  Fix one marked point
$\bq_{\bs}$.  If no first-layer activation hyperplane of $N$
contained this point, then a sufficiently small neighborhood of it
would avoid all such hyperplanes.  After discarding identically zero
preactivations, all first-layer neurons would have fixed activation
states there. The first hidden layer would therefore compute an affine transformation on that
neighborhood and could be absorbed into the affine transformation feeding the next hidden layer. This would produce a local depth-$2$ representation of $g_d$,
contradicting Lemma~\ref{lem:local-td-hardness} through Equation~\eqref{eq:transported-local-function}.

Thus, each of the $2^{d-1}$ marked points lies on some first-layer activation hyperplane. By general position, each hyperplane contains at most $d$ marked points. If the first hidden-layer width is $w$, then
\[
  dw\geq2^{d-1}.
\]
Consequently,
\[
  w\geq\frac{2^{d-1}}d.
\]
This proves the shallow lower bound.

It remains to verify the realization and regularity claims.
Item~\ref{item:perturbation-realization} of
Lemma~\ref{lem:general-position-points} and
Equation~\eqref{eq:rescaled-points-and-displacement} show that the first
hidden layer computes the ReLU features whose appropriate affine
combinations give the coordinates of $\Delta$.  Adding the $2d$
neurons $\relu{x_i},\relu{-x_i}$ allows the next hidden layer to recover
each input coordinate as $x_i=\relu{x_i}-\relu{-x_i}$.  Hence the
second hidden layer computes
\[
  \relu{(U(\bx))_i},
  \qquad
  \relu{-(U(\bx))_i},
  \qquad i=1,\ldots,d,
\]
and the third hidden layer consists of the single neuron computing
\[
  \relu{
    1-
    \sum_{i=1}^d
    \left(
      \relu{(U(\bx))_i}
      +\relu{-(U(\bx))_i}
    \right)
  }
  =
  \relu{1-\norm{U(\bx)}_1}
  =
  t_d(U(\bx))
  =
  g_d(\bx).
\]
Here the first equality uses
$|a|=\relu{a}+\relu{-a}$, the second uses the definition
$t_d(\bx)=\relu{1-\norm{\bx}_1}$, and the final equality uses
Equation~\eqref{eq:constructed-map-and-target}.  After increasing the
universal constant, the three hidden widths are at most $Cd^4,2d,1$.
Since $\Delta$ is globally CPWL, so is $U$; the displayed network
computes $g_d$ globally, so $g_d$ is also CPWL.  This proves the deep
upper bound.

The Lipschitz estimate in
Equation~\eqref{eq:rescaled-displacement-bounds} and the definition of
$U$ in Equation~\eqref{eq:constructed-map-and-target} give
\[
  \Lip(U)
  \leq
  1+\frac1{16}
  =\frac{17}{16}.
\]
Since $\Lip(t_d)=\sqrt d$, the above displayed equation and
Equation~\eqref{eq:constructed-map-and-target} give
\[
  \Lip(g_d)
  \leq
  \frac{17\sqrt d}{16}
  <2\sqrt d.
\]

Finally, Equation~\eqref{eq:rescaled-displacement-bounds} shows that
$\Delta$ is a contraction on the complete metric space $\R^d$, since its Lipschitz constant is strictly below $1$.
The Banach fixed-point theorem therefore gives a unique point
$\bx_0\in\R^d$ satisfying $\Delta(\bx_0)=\bx_0$.  By
Equation~\eqref{eq:constructed-map-and-target}, $U(\bx_0)=\bzero$.
Equation~\eqref{eq:rescaled-displacement-bounds} also gives
\[
  \norm{\bx_0}_\infty
  =
  \norm{\Delta(\bx_0)}_\infty
  \leq\frac1{32d(d-1)}
  <1.
\]
Thus $\bx_0$ lies in $(-1,1)^d$, and
Equation~\eqref{eq:constructed-map-and-target} yields
\begin{equation}
  \label{eq:target-attains-one}
  g_d(\bx_0)
  =
  t_d(U(\bx_0))
  =
  t_d(\bzero)
  =1.
\end{equation}
Set $\by\coloneqq\bone/\sqrt d$, so that $\by\in[-1,1]^d$ and
$\norm{\by}_1=\sqrt d$.  Equations~\eqref{eq:constructed-map-and-target} and
\eqref{eq:rescaled-displacement-bounds} give
\[
  \norm{U(\by)}_1
  \geq
  \sqrt d-\norm{\Delta(\by)}_1
  \geq
  \sqrt d-\frac1{32(d-1)}
  >1,
\]
and hence
\begin{equation}
  \label{eq:target-attains-zero}
  g_d(\by)
  =
  t_d(U(\by))
  =
  \relu{1-\norm{U(\by)}_1}
  =0.
\end{equation}
Equation~\eqref{eq:constructed-map-and-target} and the definition
$t_d(\bx)=\relu{1-\norm{\bx}_1}$ show that
$g_d(\R^d)\subseteq[0,1]$.  Moreover, $g_d$ is continuous, so the
image of the connected cube $[-1,1]^d$ is an interval.  By
Equations~\eqref{eq:target-attains-one} and
\eqref{eq:target-attains-zero}, this interval contains both endpoints
of $[0,1]$.  Therefore,
\[
  g_d([-1,1]^d)=[0,1].
\]
\end{proof}

\section{Proof of Theorem~\ref{thm:l2-hierarchy}}
\label{app:l2-hierarchy-proof}

\begin{lemma}[Base-function approximation lower bound]
\label{lem:base-gap}
Let
\[
  h_0(\bx)
  \coloneqq
  \relu{x_1+1}-2\relu{x_1}+\relu{x_1-1},
  \qquad
  \Dcal_0^d\coloneqq\Ucal((-1,1)^d).
\]
Then the constant function $\bx\mapsto\frac12$ is an optimal affine approximation to $h_0$ in $L_2(\Dcal_0^d)$, and
\[
  \inf_{A\,\mathrm{affine}}
  \E_{\bx\sim\Dcal_0^d}
  \left[\bigl(h_0(\bx)-A(\bx)\bigr)^2\right]
  =\frac1{12}.
\]
\end{lemma}

\begin{proof}
On $(-1,1)^d$,
\[
  h_0(\bx)=1-|x_1|.
\]
Globally, $h_0(\bx)=\relu{1-|x_1|}$, so $h_0$ takes values in
$[0,1]$ and vanishes whenever $|x_1|\geq1$.
Its graph is shown in Figure~\ref{fig:base-function}; it is a triangular
function that is constant in all directions perpendicular to $x_1$.

Let $X\sim\Dcal_0^1$.  We first reduce the $d$-dimensional
optimization to its one-dimensional counterpart.  For an arbitrary
affine function $A:\R^d\to\R$, write
$\bx'\coloneqq(x_2,\ldots,x_d)$. Under $\Dcal_0^d$, the first
coordinate and $\bx'$ are independent, with respective laws
$\Dcal_0^1$ and $\Dcal_0^{d-1}$.  For fixed $\bx'$, the
restriction $x\mapsto A(x,\bx')$ is affine.  Therefore, for any affine $A$,
\[
  \begin{aligned}
  &\E_{\bx'\sim\Dcal_0^{d-1}}
  \left[
    \E_{X\sim\Dcal_0^1}
    \left[
      \bigl(1-|X|-A(X,\bx')\bigr)^2
    \right]
  \right]\\
  &\hspace{35mm}\geq
  \inf_{a,b\in\R}
  \E_{X\sim\Dcal_0^1}
  \left[\bigl(1-|X|-aX-b\bigr)^2\right].
  \end{aligned}
\]
Taking the infimum over $A$ gives the lower-bound direction.
Conversely, for all $a,b\in\R$, the affine transformation
$\bx\mapsto ax_1+b$ has loss equal to the corresponding
one-dimensional expectation.  Taking the infimum over $a,b$ gives
the reverse direction.
Hence
\begin{equation}
  \label{eq:affine-base-reduction}
  \inf_{A\,\mathrm{affine}}
  \E_{\bx\sim\Dcal_0^d}
  \left[\bigl(h_0(\bx)-A(\bx)\bigr)^2\right]
  =
  \inf_{a,b\in\R}
  \E_{X\sim\Dcal_0^1}
  \left[\bigl(1-|X|-aX-b\bigr)^2\right].
\end{equation}

It remains to evaluate this one-dimensional infimum.  For arbitrary
$a,b\in\R$, expanding the squared error gives
\begin{align*}
  \E\left[(1-|X|-aX-b)^2\right]
  &=
  \E\left[(1-|X|-b)^2\right]
  -2a\E\left[X(1-|X|-b)\right]
  +a^2\E\left[X^2\right]\\
  &=
  \E\left[(1-|X|-b)^2\right]
  +a^2\E\left[X^2\right].
\end{align*}
The second equality holds because $1-|X|-b$ is even in $X$, whereas
$X$ is odd, so the mixed expectation vanishes by symmetry.
Thus the optimal slope is $a=0$, while the optimal constant is
\[
  b=\E\left[1-|X|\right]
  =\frac12\int_{-1}^1(1-|x|)\,dx
  =\int_0^1(1-x)\,dx
  =\frac12.
\]
Its one-dimensional squared error is
\begin{align*}
  \E\left[\left(1-|X|-\frac12\right)^2\right]
  &=
  \frac12\int_{-1}^1\left(\frac12-|x|\right)^2\,dx\\
  &=
  \int_0^1\left(\frac12-x\right)^2\,dx\\
  &=
  \left.\frac{\left(x-\frac12\right)^3}{3}\right|_{0}^{1}
  =\frac1{12}.
\end{align*}
Together with Equation~\eqref{eq:affine-base-reduction}, this proves
the claimed error, and the constant function $\frac12$ attains it.
\end{proof}

\begin{proposition}[A CPWL copying transformation]
\label{prop:one-layer-copier}
There are universal constants $C,C_0>0$ such that, for all $d\geq2$,
with
\[
  r\coloneqq2^{-C_0d^3},
\]
there are centers
$\left\{\bq_{\bs}:\bs\in\left\{-\frac12,\frac12\right\}^d\right\}$
and a vector-valued
depth-$2$ ReLU map $S:\R^d\to\R^d$ of width at most $Cd^4$ such
that:
\begin{enumerate}
  \item
  \[
    \norm{\bq_{\bs}-\bs}_\infty
    \leq\frac1{64d};
  \]

  \item \label{item:copier-protection}
  the cubes $\bq_{\bs}+[-r,r]^d$ are pairwise disjoint and
  contained in $(-1,1)^d$, and any hyperplane intersects at
  most $d$ of them;

  \item \label{item:copier-decoding}
  for all $\bs\in\left\{-\frac12,\frac12\right\}^d$ and
  $\bu\in[-1,1]^d$,
  \begin{equation}
    \label{eq:copy-identity}
    S(\bq_{\bs}+r\bu)=\bu;
  \end{equation}

  \item $\Lip(S)<\frac4r$.
\end{enumerate}
Consequently, for all functions $h:\R^d\to\R$, all
$\bs\in\left\{-\frac12,\frac12\right\}^d$, and all $\bu\in[-1,1]^d$,
\[
  (h\circ S)(\bq_{\bs}+r\bu)=h(\bu).
\]
\end{proposition}

The final identity is the key property of the transformation that we use. Precomposition by a depth-$2$ ReLU map creates $2^d$ exact copies of $h$.  If $h$ is computed by a ReLU network, the affine output layer of $S$ merges into the first affine map of the network computing $h$, so this precomposition
adds exactly one hidden layer.

\begin{proof}
Invoke Lemma~\ref{lem:general-position-points} in dimension $d$ to
obtain the small displacement map that will perturb the cube centers
into general position.  More precisely,
Item~\ref{item:perturbation-realization} supplies a map $\Delta$ and
defines the centers by
\[
  \bq_{\bs}\coloneqq\bs+\Delta(\bs),
  \qquad
  \bs\in\left\{-\frac12,\frac12\right\}^d.
\]
We first choose the side length $r$ so that no hyperplane can
meet more than $d$ of the centered cubes, and then construct $S$ so
that it maps any such cube, after rescaling, onto $[-1,1]^d$.
Define the unperturbing map
\[
  U(\bx)\coloneqq\bx-\Delta(\bx).
\]
The subtraction reverses the local displacement: near $\bq_{\bs}$,
the map $\Delta$ still equals $\Delta(\bs)$, so subtracting it recovers
the corresponding unperturbed point.
Item~\ref{item:perturbation-bounds} gives
\[
  \Lip(U)\leq1+\Lip(\Delta)\leq\frac{17}{16}.
\]
For all distinct
$\bs_0,\ldots,\bs_d\in\left\{-\frac12,\frac12\right\}^d$,
Item~\ref{item:perturbation-determinant} gives
\[
  \left|\det Q(\bs_0,\ldots,\bs_d)\right|
  \geq 2^{-C_1d^3},
\]
for a universal constant $C_1>0$, where
$Q(\bs_0,\ldots,\bs_d)$ is the $d\times d$ matrix whose
$i^\mathrm{th}$ column is $\bq_{\bs_i}-\bq_{\bs_0}$.

There is a universal constant $C_2>0$ such that, for all $d\geq2$,
\[
  2d\sqrt d\,(2\sqrt d)^{d-1}
  \leq
  2^{C_2d\log_2(d+1)}.
\]
Choose
\[
  r\coloneqq2^{-C_0d^3},
\]
where $C_0>C_1+C_2$ is also large enough that $r<\frac1{8d}$.
Since
$d\log_2(d+1)\leq d^3$ for $d\geq2$, this choice gives
\begin{equation}
  \label{eq:copier-determinant-error}
  d(2r\sqrt d)(2\sqrt d)^{d-1}
  \leq
  2^{-(C_0-C_2)d^3}
  <
  2^{-C_1d^3}.
\end{equation}
The cubes
\[
  \bq_{\bs}+[-r,r]^d,
  \qquad
  \bs\in\left\{-\frac12,\frac12\right\}^d,
\]
are pairwise disjoint and contained in $(-1,1)^d$.  Indeed, distinct
unperturbed vertices differ by one in some coordinate, whereas each
center moves by at most $\frac1{64d}$ in that coordinate.  Thus the
centers $\bq_{\bs}$ of any two distinct perturbed cubes are at $\ell_\infty$-distance
greater than $\frac34$,
while $r<\frac1{8d}$.  The same bounds put each of these cubes inside
$(-1,1)^d$.

We now argue that no hyperplane intersects more than $d$ of
these cubes.  Suppose by contradiction that one hyperplane meets cubes
indexed by distinct $\bs_0,\ldots,\bs_d$.  Choose intersection points
\[
  \bv_i\coloneqq\bq_{\bs_i}+\bxi_i,
  \qquad
  \norm{\bxi_i}_\infty\leq r,
  \qquad
  \norm{\bxi_i}_2\leq r\sqrt d,
  \qquad i=0,\ldots,d,
\]
and let $V\in\R^{d\times d}$ be the matrix whose $i^\mathrm{th}$
column is $\bv_i-\bv_0$.  Since all the $\bv_i$ lie on one hyperplane,
\[
  \det V=0.
\]
Let $Q\coloneqq Q(\bs_0,\ldots,\bs_d)$.  The displacement bound and
$\norm{\bxi_i}_\infty\leq r$ give
\[
  \norm{\bq_{\bs}}_\infty
  \leq
  \frac12+\frac1{64d}
  <1,
  \qquad
  \norm{\bv_i}_\infty
  \leq
  \frac12+\frac1{64d}+r
  <1.
\]
All columns of $Q$ and $V$ are therefore differences of two vectors
whose coordinates have absolute value less than one, and hence have
Euclidean norm at most $2\sqrt d$.

For $i=0,\ldots,d$, let $Q^{(i)}$ be the matrix whose first $i$
columns are those of $V$ and whose remaining columns are those of
$Q$.  Thus $Q^{(0)}=Q$ and $Q^{(d)}=V$, so the triangle inequality
gives
\[
  |\det Q-\det V|
  \leq
  \sum_{i=1}^d
  \left|\det Q^{(i-1)}-\det Q^{(i)}\right|.
\]
For each $i$, the only changed column is
\[
  (\bv_i-\bv_0)-(\bq_{\bs_i}-\bq_{\bs_0})
  =
  \bxi_i-\bxi_0,
  \qquad
  \norm{\bxi_i-\bxi_0}_2\leq2r\sqrt d.
\]
By multilinearity,
\[
  \det Q^{(i)}-\det Q^{(i-1)}
  =
  \det\left[
    \bv_1-\bv_0,\ldots,\bv_{i-1}-\bv_0,
    \bxi_i-\bxi_0,
    \bq_{\bs_{i+1}}-\bq_{\bs_0},\ldots,
    \bq_{\bs_d}-\bq_{\bs_0}
  \right].
\]
Hadamard's inequality and the $2\sqrt d$ bound on all other columns
now give
\[
  \left|\det Q^{(i-1)}-\det Q^{(i)}\right|
  \leq
  (2r\sqrt d)(2\sqrt d)^{d-1}.
\]
Combining these estimates with
Item~\ref{item:perturbation-determinant}, $\det V=0$, and
Equation~\eqref{eq:copier-determinant-error}, we obtain the
contradiction
\[
  2^{-C_1d^3}
  \leq
  |\det Q|
  =
  |\det Q-\det V|
  \leq
  d(2r\sqrt d)(2\sqrt d)^{d-1}
  <
  2^{-C_1d^3}.
\]

We now define the copying transformation.  Let
\[
  \tau(t)
  \coloneqq
  -\frac12+2\relu{t+\frac14}-2\relu{t-\frac14},
\]
and use the same symbol for its coordinatewise application to vectors.

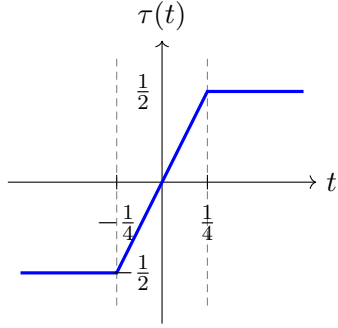
\begin{figure}[H]
  \centering
  \ifshowgraphs
  \begin{tikzpicture}[x=2.4cm,y=2.4cm]
    \draw[->] (-0.85,0) -- (0.85,0) node[right] {$t$};
    \draw[->] (0,-0.78) -- (0,0.78) node[above] {$\tau(t)$};
    \draw[densely dashed,gray] (-0.25,-0.68) -- (-0.25,0.68);
    \draw[densely dashed,gray] (0.25,-0.68) -- (0.25,0.68);
    \draw[very thick,blue]
      (-0.78,-0.5) -- (-0.25,-0.5) -- (0.25,0.5) -- (0.78,0.5);
    \draw (-0.25,0.04) -- (-0.25,-0.04)
      node[below=2pt] {$-\frac14$};
    \draw (0.25,0.04) -- (0.25,-0.04)
      node[below=2pt] {$\frac14$};
    \node[left] at (0,0.5) {$\frac12$};
    \node[left] at (0,-0.5) {$-\frac12$};
  \end{tikzpicture}
  \fi
  \caption{The scalar clipping function $\tau$.  It equals
  $-\frac12$ on $\left(-\infty,-\frac14\right]$, equals $\frac12$ on
  $\left[\frac14,\infty\right)$, and interpolates linearly between the
  two plateaus.  The two axes use the same scale.}
  \label{fig:copier-clipping-function}
\end{figure}

Set
\[
  S(\bx)
  \coloneqq
  \frac{U(\bx)-\tau(\bx)}r.
\]
For all $\bs\in\left\{-\frac12,\frac12\right\}^d$ and
$\bu\in[-1,1]^d$, Item~\ref{item:perturbation-local-constancy} of
Lemma~\ref{lem:general-position-points} gives
\[
  \Delta(\bq_{\bs}+r\bu)
  =\Delta(\bs),
\]
because $\norm{r\bu}_\infty\leq r<\frac1{8d}$.  Moreover,
\[
  |(\bq_{\bs}+r\bu)_i|
  \geq
  \frac12-\frac1{64d}-r
  >
  \frac14,
\]
and its sign is the sign of $s_i$.  Hence
\[
  \tau(\bq_{\bs}+r\bu)=\bs.
\]
Using Item~\ref{item:perturbation-realization} once more, we obtain
\[
  \begin{aligned}
    U(\bq_{\bs}+r\bu)
    &=
    \bq_{\bs}+r\bu-\Delta(\bq_{\bs}+r\bu)\\
    &=
    \bs+\Delta(\bs)+r\bu-\Delta(\bs)\\
    &=
    \bs+r\bu.
  \end{aligned}
\]
Consequently,
\[
  S(\bq_{\bs}+r\bu)
  =
  \frac{U(\bq_{\bs}+r\bu)-\tau(\bq_{\bs}+r\bu)}r
  =
  \frac{\bs+r\bu-\bs}{r}
  =
  \bu.
\]
It follows that, for all functions $h$,
\[
  (h\circ S)(\bq_{\bs}+r\bu)=h(\bu).
\]

It remains to verify that the claimed architecture can be implemented using a single hidden layer with the claimed width. In one common hidden layer, concatenate three feature blocks: the $\Ocal(d^4)$ ReLU features
used by the depth-$2$ realization of $\Delta$; the $2d$ neurons
\[
  \relu{x_i+\frac14},\quad \relu{x_i-\frac14},
  \qquad i\in[d],
\]
that compute the $d$ coordinates of $\tau(\bx)$; and the $2d$ identity
carriers
\[
  \relu{x_i},\quad\relu{-x_i},
  \qquad i\in[d],
\]
from which the output layer recovers
$x_i=\relu{x_i}-\relu{-x_i}$.  All neurons in these three blocks
apply a ReLU to an affine transformation of the same input $\bx$, so all
blocks are evaluated in parallel by this single hidden layer.  Its
output layer computes, for all $i\in[d]$,
\[
  S_i(\bx)
  =
  \frac{x_i-\Delta_i(\bx)-\tau(x_i)}r.
\]
The total width is $\Ocal(d^4)+4d=\Ocal(d^4)$.  Finally, the
coordinatewise map $\tau$ is $2$-Lipschitz, and therefore
\[
  \Lip(S)
  \leq
  \frac{\Lip(U)+\Lip(\tau)}r
  <\frac4r.
\]
\end{proof}

\begin{proof}[Proof of Theorem~\ref{thm:l2-hierarchy}]
Fix $r$, $\{\bq_{\bs}\}$, and $S$ supplied by
Proposition~\ref{prop:one-layer-copier}, and let $h_0$ and
$\Dcal_0^d$ be the base pair from
Lemma~\ref{lem:base-gap}.  Let $\bX_0\sim\Dcal_0^d$.  For
$j=1,\ldots,k-1$, set
\[
  h_j\coloneqq h_{j-1}\circ S.
\]
Define the matching random vector $\bX_j$ recursively: at each step,
independently of $\bX_{j-1}$, choose $\bs$ uniformly from
$\left\{-\frac12,\frac12\right\}^d$ and set
$\bX_j\coloneqq\bq_{\bs}+r\bX_{j-1}$.  Equivalently, for all bounded
measurable functions $F$,
\begin{equation}
  \label{eq:distribution-recursion}
  \E[F(\bX_j)]
  =
  2^{-d}\sum_{\bs\in\left\{-\frac12,\frac12\right\}^d}
  \E\left[F(\bq_{\bs}+r\bX_{j-1})\right].
\end{equation}
Here $j$ counts how many copying transformations are involved: $h_j$ is the target after $j$ copying steps, while $\bX_j$ makes the same $j$ choices of copies.
Item~\ref{item:copier-decoding} of
Proposition~\ref{prop:one-layer-copier} gives, for all
$\bu\in[-1,1]^d$,
\begin{equation}
  \label{eq:recursive-copy-identity}
  h_j(\bq_{\bs}+r\bu)
  =h_{j-1}\bigl(S(\bq_{\bs}+r\bu)\bigr)
  =h_{j-1}(\bu).
\end{equation}
Thus each recursion step replaces the preceding target-distribution
pair by $2^d$ exact, equally weighted copies of it that are slightly perturbed.

We first prove the separation for these undilated pairs whose support
is inside the unit hypercube. We claim that, for all
$j=0,\ldots,k-1$, all depth-$(j+1)$
networks $N$
of width at most $w$, where $wd\leq2^d$, satisfy
\begin{equation}
  \label{eq:inductive-lower-bound}
  \E\left[\bigl(N(\bX_j)-h_j(\bX_j)\bigr)^2\right]
  \geq
  \frac1{12}
  \left(1-\frac{wd}{2^d}\right)^j.
\end{equation}
At $j=0$, the network has depth $1$ and therefore no hidden layers, so
it computes an affine transformation.  The claim is thus exactly
Lemma~\ref{lem:base-gap}.

For the induction step, suppose the claim holds at $j-1$ and let $N$
have depth $j+1$.  Its first layer has at most $w$ activation
hyperplanes.  By Item~\ref{item:copier-protection} of
Proposition~\ref{prop:one-layer-copier}, each such hyperplane meets at
most $d$ of the $2^d$ cubes.  Consequently, at least
$2^d-wd$ of the outer cubes $\bq_{\bs}+r(-1,1)^d$, each carrying one
copy of $h_{j-1}$, are missed by all first-layer hyperplanes.

On a missed cube, all non-constant first-layer preactivations have fixed
signs. A constant preactivation also gives an affine output: it is either
constant or identically zero. Hence every first-layer neuron computes an affine function
throughout that cube, and the entire first hidden layer can be absorbed into
the affine map computed by the subsequent layer.
For the corresponding
$\bs$, the function
\[
  \bu\longmapsto N(\bq_{\bs}+r\bu)
\]
therefore agrees on $(-1,1)^d$ with a depth-$j$,
width-at-most-$w$ network.  The induction hypothesis and
Equation~\eqref{eq:recursive-copy-identity} imply that its conditional
squared error is at least
\[
  \frac1{12}
  \left(1-\frac{wd}{2^d}\right)^{j-1}.
\]
The distribution in Equation~\eqref{eq:distribution-recursion} gives
each outer copy mass $2^{-d}$.  Keeping the errors on the missed copies
and discarding the non-negative errors on all other copies yields
\[
  \E\left[\bigl(N(\bX_j)-h_j(\bX_j)\bigr)^2\right]
  \geq
  \frac{2^d-wd}{2^d}\cdot\frac1{12}
  \left(1-\frac{wd}{2^d}\right)^{j-1}
  =
  \frac1{12}
  \left(1-\frac{wd}{2^d}\right)^j.
\]
This proves Equation~\eqref{eq:inductive-lower-bound}.

We now dilate the undilated pair to obtain the pair stated in the
theorem.  Since $h_0$ has depth $2$, width $3$, and Lipschitz constant
one, $h_{k-1}$ has
depth-$(k+1)$, width $\Ocal(d^4)$, size
$\Ocal(k d^4)$, and
\[
  \Lip(h_{k-1})
  \leq
  \left(\frac4r\right)^{k-1}.
\]
Let
\[
  \alpha\coloneqq
  4^{k-1}
  r^{-(k-1)},
  \qquad
  f_{k,d}(\bx)
  \coloneqq
  h_{k-1}
  \left(\frac{\bx}{\alpha}\right).
\]
We now specify the distribution
$\Dcal_{k}^d$ appearing in
Theorem~\ref{thm:l2-hierarchy}: sample
$\bX_{k-1}$ as above and output
$\alpha\bX_{k-1}$.  The input dilation
changes neither the
depth, width, nor size of the realizing network.  Since $h_0$ takes
values in $[0,1]$, so do all $h_j$, and hence so does
$f_{k,d}$.
Moreover,
\[
  \Lip(f_{k,d})
  \leq
  \frac1\alpha
  \left(\frac4r\right)^{k-1}
  =1.
\]

The same dilation leaves squared approximation error unchanged; this
is also the input-dilation principle of
\citet[Theorem~9]{safran-eldan-shamir-2019}.  Explicitly, for all
networks $N$,
\begin{equation}
  \label{eq:dilation-invariance}
  \E_{\bx\sim\Dcal_{k}^d}
  \left[\bigl(N(\bx)-f_{k,d}(\bx)\bigr)^2\right]
  =
  \E\left[
    \bigl(
      N(\alpha\bX_{k-1})
      -h_{k-1}
       (\bX_{k-1})
    \bigr)^2
  \right].
\end{equation}

Let $N$ now be a
depth-$k$ network satisfying the width bound
in the theorem.  Bernoulli's inequality gives
\[
  \left(1-\frac{wd}{2^d}\right)^{k-1}
  \geq
  1-(k-1)\frac{wd}{2^d}
  \geq\frac12.
\]
The network $\bu\mapsto N(\alpha\bu)$ has the same depth and width as
$N$.  Equations~\eqref{eq:inductive-lower-bound} and
\eqref{eq:dilation-invariance} therefore give
\[
  \E_{\bx\sim\Dcal_{k}^d}
  \left[\bigl(N(\bx)-f_{k,d}(\bx)\bigr)^2\right]
  \geq\frac1{24}.
\]
This proves the two depth-separation claims.

We finish by verifying the additional distributional properties
stated after the theorem.  None of them is used in the lower-bound
induction above.  The number, size, and location of the supporting
cubes make the geometry and exponential spread of the distribution
explicit.  The conditional-density bound is recorded specifically for
comparison with the Vardi--Shamir barrier
\citep{vardi-shamir-2020}.

Using induction in Equation~\eqref{eq:distribution-recursion} shows that
the law of $\bX_j$ is uniform on $2^{dj}$ pairwise disjoint open cubes
of side length $2r^j$, all contained in $(-1,1)^d$.  Indeed, at the next step
the map $\bx\mapsto\bq_{\bs}+r\bx$ shrinks all old cubes by $r$;
mixture components with different $\bs$ lie in the disjoint cubes
$\bq_{\bs}+r(-1,1)^d$, while for each fixed $\bs$, the images of the old cubes remain disjoint. The equal mixture weights and equal Jacobians preserve uniformity.

To locate these cubes relative to the origin, unroll the recursion.
For $j\geq1$, all $\bx$ in the support of the law of $\bX_j$ have the form
\[
  \bx
  =
  \bq_{\bs_1}+r\bq_{\bs_2}+\cdots+r^{j-1}\bq_{\bs_j}+r^j\bu
\]
for some $\bs_1,\ldots,\bs_j\in
\left\{-\frac12,\frac12\right\}^d$ and $\bu\in[-1,1]^d$.
All coordinates of $\bq_{\bs_1}$ have magnitude at least
$\frac12-\frac1{64d}$, whereas
$\norm{\bq_{\bs_i}}_\infty<\frac58$ for all $i=2,\ldots,j$.
Since
$r<\frac1{8d}\leq\frac1{16}$ and $r^j\leq r$,
\[
  \norm{\bx}_\infty
  \geq
  \frac12-\frac1{64d}-\frac{5r}{8(1-r)}-r^j
  >
  \frac12-\frac1{128}-\frac1{24}-\frac1{16}
  >\frac18.
\]
The final dilation turns each cube supporting the law of
$\bX_{k-1}$
into a cube of side length
\[
  2\alpha r^{k-1}
  =2\cdot4^{k-1}.
\]
Because the undilated support lies in $[-1,1]^d$ and outside
$[-\frac18,\frac18]^d$, the support of
$\Dcal_{k}^d$ satisfies
\[
  \inf_{\bx\in\operatorname{supp}
    (\Dcal_{k}^d)}\norm{\bx}_2
  \geq\frac\alpha8,
  \qquad
  \sup_{\bx\in\operatorname{supp}
    (\Dcal_{k}^d)}\norm{\bx}_2
  \leq\sqrt d\,\alpha.
\]
After adjusting universal constants, these bounds are respectively
$2^{c(k-1)d^3}$ and
$2^{C(k-1)d^3}$. For any coordinate and almost all fixed values of the other $d-1$
coordinates, the line obtained by varying that coordinate intersects some $m\geq1$ supporting cubes in disjoint intervals, each of length $2\cdot4^{k-1}$.  Its conditional law is uniform on that union, so its density is
\[
  \frac1{2\cdot4^{k-1}
    m}
  \leq\frac18.
\]
This completes the proof.
\end{proof}

\end{document}